%% file: main.tex
\documentclass[runningheads]{llncs}
\usepackage[T1]{fontenc}
\usepackage{graphicx}
\usepackage{xcolor}
\usepackage{xurl}         
\usepackage{hyperref}

\input{preamble.tex}

\begin{document}
\title{Reinforcement Learning with Stochastic Reward Machines\thanks{A shorter version of this paper appeared in the Proceedings of the Thirty-Sixth AAAI Conference on Artificial Intelligence (AAAI-22)~\cite{srmi}. Source code available at \url{https://github.com/corazza/srm}.}}
%
%

\author {
    Jan Corazza\textsuperscript{\rm 1,2}\hspace*{0.15cm}
    Ivan Gavran\textsuperscript{\rm 2}\hspace*{0.15cm}
    Daniel Neider\textsuperscript{\rm 2}
}
\authorrunning{J. Corazza et al.}
\institute {
    corazzajan@gmail.com\hspace*{0.15cm}
    gavran@mpi-sws.org\hspace*{0.15cm}
    neider@mpi-sws.org\\ \vspace*{0.15cm}
    \textsuperscript{\rm 1} University of Zagreb\\
    \textsuperscript{\rm 2} Max Planck Institute for Software Systems
}

\maketitle              
\begin{abstract}
    \input{abstract.tex}

\end{abstract}
%
%
%

\input{sections/1_intro}
\input{sections/2_background}

\input{sections/3_noisemodel}
\input{sections/4_methods}
\input{sections/5_results}
\input{sections/6_conclusion}


%
%
%
\FloatBarrier
\bibliographystyle{splncs04}
\bibliography{references}
%

\clearpage
\appendix
\input{appendix_sections/1_additional}

\input{appendix_sections/2_asymmetric}
\input{appendix_sections/3_proofs}

\end{document}

%% file: preamble.tex
\usepackage{placeins}   
\usepackage{color}
\usepackage{subcaption}
\usepackage{enumitem}
\usepackage{amsmath}
\usepackage{amsfonts}
\usepackage[ruled,vlined,linesnumbered]{algorithm2e}
\usepackage[utf8]{inputenc}

\DeclareGraphicsExtensions{.pdf,.png,.jpg,.jpeg}

\usepackage{pgfplots}
\DeclareUnicodeCharacter{2212}{−}
\usepgfplotslibrary{groupplots,dateplot}
\pgfplotsset{compat=newest}

\usepackage{tikz}
\usetikzlibrary{patterns,shapes.arrows, matrix}
\usetikzlibrary{decorations.markings}
\usepackage{tikzit}
\input{style.tikzstyles}

\newtheorem{assumption}{Assumption}
\newtheorem{task}{Task}

\input{macros}

%% file: style.tikzstyles
\tikzstyle{normal_node}=[fill=white, draw=black, shape=circle]

\tikzstyle{normal_edge}=[->, draw={rgb,255: red,64; green,64; blue,64}]

%% file: macros.tex
\newcommand{\set}[1]{\{ #1 \}}
\newcommand{\reals}{\mathbb{R}}
\newcommand{\nat}[0]{\mathbb{N}\xspace}

\xspace

\newcommand{\init}{I}
\newcommand{\term}{T}
\newcommand{\jirp}{JIRP\xspace}
\newcommand{\sjirp}{SRMI\xspace}

\newcommand{\setOfQFunctions}{Q}
\newcommand{\qFunction}{q}

\newcommand{\rmLabels}{P}
\newcommand{\rmLabelingFunction}{L}
\newcommand{\rmInputAlphabet}{2^\rmLabels}

\newcommand{\machine}{\fsm{A}}
\newcommand{\machineB}{\fsm{B}}

\newcommand{\mealyStates}{V}

\newcommand{\mealyCommonState}{\fsm{v}}
\newcommand{\mealyCommonStateP}{\fsm{p}}
\newcommand{\mealyCommonStateQ}{\fsm{q}}
\newcommand{\mealyCommonStatePrime}{\mealyCommonState '}

\newcommand{\mealyInputAlphabet}{\ensuremath{{2^\rmLabels}}}
\newcommand{\mealyInit}{{\mealyCommonState_\init}}
\newcommand{\mealyTerm}{{\mealyCommonState_\term}}

\newcommand{\mealyOutputAlphabet}{\ensuremath{O}}
\newcommand{\mealyOutput}{\sigma}
\newcommand{\mealyTransition}{\delta}
\newcommand{\inputTrace}{\lambda}

\newcommand{\outputTrace}{\rho}

\newcommand{\setOfSeenRewards}{R}
\newcommand{\setOfSeenTraces}{A}
\newcommand{\setOfConsistentTraces}{C}

\newcommand{\hypothesisRM}{\mathcal H}
\newcommand{\hypothesisFinal}{\hypothesisRM^f}
\newcommand{\hypothesisFinalSeq}{\hypothesisFinal_1 \hypothesisFinal_2 \cdots}
\newcommand{\hypothesisSeq}{\hypothesisRM_1 \hypothesisRM_2 \cdots}
\newcommand{\hypothesisRMG}{\mathcal G}
\newcommand{\hypothesisRMZ}{\mathcal Z}
\newcommand{\environmentRM}{\mathcal T}

\newcommand{\counterexamples}{\ensuremath{X}}

\newcommand{\mdp}{\mathcal{M}}
\newcommand{\mdpStates}{S}
\newcommand{\mdpCommonState}{s}
\newcommand{\mdpCommonStatePrime}{\mdpCommonState '}
\newcommand{\mdpInit}{\mdpCommonState_\init}
\newcommand{\mdpActions}{A}
\newcommand{\mdpCommonAction}{a}
\newcommand{\mdpCommonActionPrime}{\mdpCommonAction '}
\newcommand{\mdpProb}{p}
\newcommand{\mdpRewardFunction}{R}

\newcommand{\mdpDiscount}{\gamma}
\newcommand{\mdpTrajectory}{\zeta}
\newcommand{\mdpLabel}{\ell}
\newcommand{\mdpLabelPrime}{\mdpLabel '}
\newcommand{\mdpRewards}{\mathcal{R}}
\newcommand{\mdpCommonReward}{r}
\newcommand{\mdpCommonRewardPrime}{\mdpCommonReward '}
\newcommand{\mdpCommonRewardHat}{\hat{\mdpCommonReward}}
\newcommand{\trajectory}[1]{\ensuremath{\mdpCommonState_0 \mdpCommonAction_1 \mdpCommonState_1\ldots \mdpCommonAction_#1 \mdpCommonState_{#1}}}
\newcommand{\labelSequence}[1]{\ensuremath{\mdpLabel_1 \mdpLabel_2\ldots \mdpLabel_{#1}}}
\newcommand{\labelSequencePrime}[1]{\ensuremath{\mdpLabel_1 ' \mdpLabel_2 ' \ldots \mdpLabel_{#1} ' }}
\newcommand{\rewardSequence}[1]{\ensuremath{\mdpCommonReward_1 \mdpCommonReward_2\ldots \mdpCommonReward_{#1}}}
\newcommand{\rewardSequencePrime}[1]{\ensuremath{\mdpCommonReward_1 ' \mdpCommonReward_2 ' \ldots \mdpCommonReward_{#1} ' }}

\newcommand{\outputDistribution}{d}
\newcommand{\distributionSequence}[1]{\ensuremath{\outputDistribution_1 \outputDistribution_2 \ldots \outputDistribution_{#1}}}

\newcommand{\mdpTransitionTriplet}{\ensuremath{(\mdpCommonState, \mdpCommonAction, \mdpCommonState ')}}

\newcommand{\policy}{\ensuremath{\pi}}

\newcommand{\commonReward}{r}
\newcommand{\maxLengthEpisode}{m}

\newcommand{\fsm}[1]{\mathsf{#1}}

\newcommand{\Pref}{\mathit{Pref}}

{
    \begin{tikzpicture}[shorten >=1pt,node distance=2cm,on grid,auto, thick]
        }%
        {
    \end{tikzpicture}
}

\newcommand{\distributionClosed}[2]{\ensuremath{D([#1, #2])}}
\newcommand{\distributionClosedN}[3]{\ensuremath{D_{#1}([#2, #3])}}

\newcommand{\uniformClosed}[2]{\ensuremath{U([#1, #2])}}

\newcommand{\dispersionParameter}{\epsilon_c}

\newcommand{\epsconsistent}{\ensuremath{\dispersionParameter}-consistent}

\newcommand{\mean}{\mu}
\newcommand{\estimatesSet}{r}
\newcommand{\meanEstimate}{\mean '}
\DeclareMathOperator*{\E}{\mathbb{E}}
\newcommand{\equivExpect}{\sim_{\E}}
\newcommand{\equivGraph}{\sim}
\newcommand{\isomorphism}{\ensuremath{I}}
\newcommand{\isomorphismFix}{\ensuremath{f_I}}
\newcommand{\trace}{\ensuremath{(\inputTrace, \outputTrace)}}

\newcommand{\easyCx}{\emph{1}\xspace}
\newcommand{\hardCx}{\emph{2}\xspace}

\newcommand{\pictureScale}{0.6}
\newcommand{\nodeScale}{0.8}

\newcommand{\formula}{\Phi}
\newcommand{\rmSize}{n}
\newcommand{\smtFormula}{\formula_{\rmSize}^{\counterexamples, \dispersionParameter}}


%% file: abstract.tex

Reward machines are an established tool for dealing with reinforcement learning problems in which rewards are sparse and depend on complex sequences of actions.
However, existing algorithms for learning reward machines assume an overly idealized setting where rewards have to be free of noise.
To overcome this practical limitation, we introduce a novel type of reward machines, called stochastic reward machines, and an algorithm for learning them.
Our algorithm, based on constraint solving, learns minimal stochastic reward machines from the explorations of a reinforcement learning agent.
This algorithm can easily be paired with existing reinforcement learning algorithms for reward machines and guarantees to converge to an optimal policy in the limit.
We demonstrate the effectiveness of our algorithm in two case studies and show that it outperforms both existing methods and a naive approach for handling noisy reward functions.


%% file: sections/1_intro.tex

\section{Introduction}
\label{sec:intro}

The key assumption of a reinforcement learning (RL) model is that the reward function is \emph{Markovian}: 
the received reward depends only on the agent's immediate state and action.
For many practical RL tasks, however, 
the most natural conceptualization of the state-space is the one in which the reward function depends on the history of actions that the agent has performed.
(Those are typically the tasks in which the agent is rewarded for complex behaviors over a longer period.)

An emerging tool used for reinforcement learning in environments with such non-Markovian rewards are \emph{reward machines}.
A reward machine is an automaton-like structure which augments the state space of the environment, capturing the temporal component of rewards.
It has been demonstrated that Q-learning~\cite{sutton2018reinforcement}, a standard RL algorithm, can be adapted to use and benefit from reward machines~\cite{icarte_rms}.

Reward machines are either given by the user, or inferred by the agent on the fly~\cite{DBLP:conf/aaai/GaonB20,DBLP:conf/aaai/Furelos-BlancoL20,jirp}.
The used learning methods ensure that the inferred machine is minimal, enabling quick optimal convergence.
Besides a faster convergence, learning minimal reward machines contributes to the interpretability of problems with an unclear reward structure.

Reward machines only model deterministic rewards.
When the machine is not known upfront, existing learning methods prove counterproductive in the presence of noisy rewards,
as there is either no reward machine consistent with the agent's experience, or the learned reward machine explodes in size, overfitting the noise.

In this paper, we introduce the notion of a \emph{stochastic reward machine},
which can capture noisy, non-Markovian reward functions, together with a novel algorithm for learning them.
The algorithm is an extension of the constraint-based formulation of Xu et al.~\cite{jirp}.
The extension relies on the parameters of the reward's distribution, making sure that experiential rewards respect the distribution. 
In every iteration, if the agent establishes that its current hypothesis about the machine is wrong, 
it updates the hypothesis (either by fixing the machine's parameters or by solving the constraint problem and learning a new machine).

While one could use the proposed algorithm to learn a suitable (non-stochas-tic) reward machine and use that machine to guide the reinforcement learning process, we recognize the value of modeling stochasticity explicitly.
First, it reveals information about the distribution of rewards, improving interpretability of the problem at hand.
Second, a direct correspondence between the stochastic reward function and the stochastic reward machine that models it makes the exposition clearer.


In our experimental evaluation,
we demonstrate the successful working of our algorithm on two noisy, non-Markovian case studies.
We compare our algorithm with existing methods (which do not deal explicitly with noise) on the same case studies:
as expected, disregarding the noise by using existing inference algorithms for classical RMs performs substantially worse than our new approach.
Finally, we compare our algorithm to a baseline method that 
tries to ``average out'' the noise.

To summarize, in this paper we
1) introduce stochastic reward machines,
2) present a novel algorithm for learning stochastic reward machines by a RL agent, 
and
3) experimentally demonstrate the efficacy of our algorithm.

\subsection{Related Work}

Using finite state machines to capture non-Markovian reward functions has been proposed already in the early work on the topic~\cite{DBLP:conf/aaai/BacchusBG96}.
Toro Icarte et al.~\cite{icarte_rms,DBLP:journals/corr/abs-2010-03950} introduced reward machines (known as \emph{Mealy machines} in other contexts) as a suitable formalism
and an algorithm that takes advantage of the reward machine structure.
Similar formalisms, including temporal logics, have been proposed by others, too~\cite{DBLP:conf/nips/JothimuruganAB19,DBLP:conf/aaai/BrafmanGP18,DBLP:conf/ijcai/CamachoIKVM19}.
This line of work assumes the reward machine to be given.

The assumption of the user-provided reward machine has been lifted in the follow-up works~\cite{DBLP:conf/aaai/GaonB20,jirp,DBLP:conf/aaai/Furelos-BlancoL20}.
Learning temporal representations of the reward has been explored in different contexts:
for multi-agent settings~\cite{neary2020reward},
for reward shaping~\cite{DBLP:conf/aaai/VelasquezBBBAMA21},
or with user-provided advice~\cite{DBLP:conf/aaai/NeiderGGT0021}.
All these approaches are fragile in presence of noisy rewards.
Other approaches focus on learning attainable sequences of labels~\cite{DBLP:conf/aaai/HasanbeigJAMK21,IcarteNIPS2019}, disregarding reward values.
If reward noise ranges over a finite set of values, Velasquez et al.~\cite{velasquez2021learning} propose active learning of reward machine-like automata with a probabilistic transition function.

Outside of the non-Markovian context, many works studied noise in rewards.
Everitt et al.~\cite{DBLP:conf/ijcai/EverittKOL17} give an overview of potential sources of noise/corruption and provide the impossibility result for learning under arbitrarily corrupted rewards.
Romoff et al.~\cite{DBLP:conf/corl/Romoff0PFP18} propose learning a reward estimator function alongside the value function.
Wang, Liu, and Li \cite{DBLP:conf/aaai/WangLL20} consider a problem of rewards being flipped according to a certain distribution.
While all these works consider much richer noise models, they are not readily applicable to the non-Markovian rewards setting.

%
%

%% file: sections/2_background.tex

\section{Preliminaries}

\label{sec:background}

This section introduces the necessary background on reinforcement learning and the formalism of reward machines for capturing non-Markovian rewards. We illustrate all notions on a running example called \emph{Mining}.
Mining, inspired by variations of Minecraft~\cite{andreas2017modular}, models the problem of finding and exploiting ore in an unknown environment.
We use this example throughout the paper.

Fig.~\ref{figure:mining-map} shows the Mining world.
An agent moves in a bounded grid-world intending to find valuable ore, gold (\texttt{G}) and platinum (\texttt{P}), and bring it to the marketplace (\texttt{M}).
Furthermore, the agent's success depends on the purity of the ore and the market prices, which is stochastic and cannot be influenced by the agent (though the spread of the market prices can naturally be bounded).
In order to do so successfully, the agent has to fulfill specific additional requirements, such as finding equipment (\texttt{E}) beforehand and not stepping into traps (\texttt{T}).

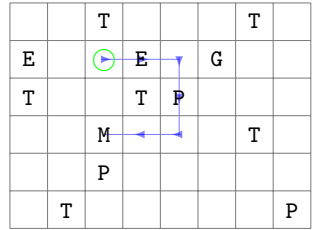
\begin{figure}[th]
    \centering
    \input{figures/mine_map.tex}
    \caption{A simplified example of the Mining environment grid and a trajectory. The agent's initial state is shown as a circle. Cells display their state labels, or are blank if the label is $\emptyset$. The trajectory indicated by arrows shows the agent collecting the equipment (\texttt{E}), finding platinum (\texttt{P})  and bringing it to the marketplace (\texttt{M}).}
    \label{figure:mining-map}
\end{figure}

In reinforcement learning, an agent learns to solve such tasks through repeated, often episodic interactions with an environment.
After each step, the agent receives feedback (a reward $\mdpCommonReward \in \mdpRewards \subset \reals$) based on its performance and acts to maximize a (discounted) sum of received rewards.
This interaction forms a stochastic process: while the environment is in some state $\mdpCommonState \in \mdpStates$, the agent chooses an action $\mdpCommonAction \in \mdpActions$ according to a policy $\policy(\mdpCommonState, \mdpCommonAction)$ (a function mapping states to probability distributions over the action space), causing the environment to transition into the next state $\mdpCommonState ' \in \mdpStates$ and giving the agent a reward $\mdpCommonReward$ (where $\mdpStates$ is the state space and $\mdpActions$ is the action space of the environment).
A realization of this process is a \emph{trajectory} $\trajectory{k}$ (optionally, rewards may be included in this sequence).
The agent continually updates its policy (i.e., learns) based on received rewards.

A reinforcement learning environment is formalized as a \emph{Markov decision process} (MDP).
As is common in the context of reward machines, we equip our MDPs with a labeling function that maps transitions to labels.
These labels correspond to high-level events that are relevant for solving the given task (e.g., finding gold, indicated by \texttt{G}).



\begin{definition}
    \label{def:mdp}
    A (labeled) \textbf{Markov decision process} is a tuple $\mdp = (\mdpStates, \mdpInit, \mdpActions,\allowbreak \mdpProb, \rmLabels, \rmLabelingFunction)$
    consisting of a finite state space $\mdpStates$, 
    an agent's initial state $\mdpInit \in \mdpStates$, 
    a finite set $\mdpActions$ of actions, and a probabilistic transition function $\mdpProb \colon \mdpStates \times \mdpActions \times \mdpStates \to [0, 1]$. 
    Additionally, a finite set $\rmLabels$ of propositional variables, and a labeling function $\rmLabelingFunction \colon \mdpStates \times \mdpActions \times \mdpStates \to \rmInputAlphabet$ determine the set of relevant high-level events that the agent senses in the environment.
    We define the size of an MDP $\mdp$, denoted as $|\mdp|$, to be the cardinality of the set $\mdpStates$  (i.e., $|\mdp| = |\mdpStates|$).
\end{definition}

Let us briefly illustrate this definition.
The agent always starts in state $\mdpInit$. 
It  then interacts with the environment by taking an action at each step.
If the agent at state $\mdpCommonState \in \mdpStates$ takes the action $a \in \mdpActions$,
its next state will be $\mdpCommonState' \in \mdpStates$ with probability $\mdpProb(\mdpCommonState, a, \mdpCommonState')$.
For this transition $(\mdpCommonState, a, \mdpCommonState')$, a labeling function $\rmLabelingFunction$ emits a set of high-level propositional variables (i.e., a label).
One can think of these labels as knowledge provided by the user. 
If the user can not provide any labeling function, $\rmLabelingFunction$ can simply return the current transition.

The Mining example can be modeled by the MDP in which states are fields of the grid world, and the agent's actions are moving in the four cardinal directions.
The transition function $\mdpProb$ models the possibility of the agent slipping when making a move.
The propositional variables used for labeling are $\texttt{E}$ (equipment found), $\texttt{P}$ (platinum found), $\texttt{G}$ (gold found), $\texttt{M}$ (marketplace reached), $\texttt{T}$ (fell into a trap); $\emptyset$ signifies that no relevant events have occurred.

To capture an RL problem fully, we need to equip an MDP with a reward function. 
Formally, a reward function $\mdpRewardFunction$ maps trajectories from $(\mdpStates \times \mdpActions)^+ \times \mdpStates$ to a cumulative distribution function over the set of possible rewards $\mdpRewards$.
For labeled MDPs, the reward function typically depends on finite sequences of observed labels rather than on the more low-level sequences of state-action pairs.

Let us emphasize the significance of defining the reward function $\mdpRewardFunction$ over finite sequences of states and actions.
Using the entire history as the argument enables us to reward behaviors that respect temporal relations and persistence naturally.
For instance, in the Mining example, the goal is to accomplish the following steps: (1) find equipment, (2) exploit a mine, (3) deliver the ore to the specified location, all while avoiding traps.
Note that the order in which the agent performs these steps is crucial:
finding ore and going to the marketplace without first picking up equipment is not rewarded; stepping into traps ends the episode with reward zero.
Reward functions that make use of the agent's entire exploration (as opposed to only the current state and action) have first been studied by Bacchus, Boutilier, and
Grove \cite{DBLP:conf/aaai/BacchusBG96} and are termed \emph{non-Markovian} reward functions.

Toro Icarte et al.~\cite{icarte_rms} have shown that \emph{reward machines (RMs)} are a powerful formalism for representing non-Markovian reward functions.
Intuitively, one can view the role of reward machines as maintaining the sufficient amount of memory to turn the non-Markovian reward function back into a ordinary, Markovian one.
This results in an important feature of RMs: they enable using standard RL algorithms (which would otherwise not be usable with non-Markovian reward functions).
Furthermore, taking advantage of the structure present in RMs, the algorithms can be made more efficient.

On a technical level, RMs are finite-state automata that transduce a sequence $\labelSequence{k}$ of labels into a sequence $\rewardSequence{k}$ of rewards.
For the sake of brevity, we here omit a formal definition and introduce the concept of RMs using an example.
To this end, let us consider the RM $\machine$ in Fig.~\ref{subfig:classical-RM}, which \emph{attempts} to capture the reward function of the Mining example and operates as follows.

Starting from the initial state $\mealyInit$, the machine transitions to an intermediate state $\mealyCommonState_1$ upon finding equipment (indicated by formula\footnote{We use propositional formulas to succinctly describe sets of labels. For instance, the formula $p \lor q$ over $\mathcal P = \{ p, q \}$ corresponds to the set $\set{\set{p}, \set{q}, \set{p, q}}$.} \texttt{E}).
From there, $\machine$ either moves to state $\mealyCommonState_2$ (upon finding platinum) or to state $\mealyCommonState_3$ (upon finding gold).
The reward, however, is delayed until the agent reaches the marketplace (indicated by the label \texttt{M}) and $\machine$ transitions to the terminal state $\mealyTerm$.
Once this happens, the machine outputs a reward of 1 (if the agent has previously collected gold) or a reward of 1.1 (if the agent has collected platinum).
By contrast, violating the prescribed order, failing to reach the marketplace, or stepping onto a trap results in no reward for the agent.
For the label sequence $(\emptyset, \set{\texttt{E}}, \emptyset, \set{\texttt{P}}, \emptyset, \emptyset, \set{\texttt{M}})$ (from Fig.~\ref{figure:mining-map}), the machine $\machine$ will produce the reward sequence $(0,0,0,0,0,0,1)$.

Note, however, that the RM in Fig.~\ref{subfig:classical-RM} fails to capture the stochastic nature of rewards in the Mining example, which stems from varying purity of the ore and market fluctuations.
This problem arises from an intrinsic property of reward machines: they only allow outputs to be real numbers and, hence, can only encode \emph{deterministic} reward functions.
This observation shows that reward machines, as currently defined and used in the literature, cannot capture the common phenomenon of noisy rewards!
In the next section, we show how to generalize the model of reward machines in order to overcome this limitation.


%

\begin{figure}[t]
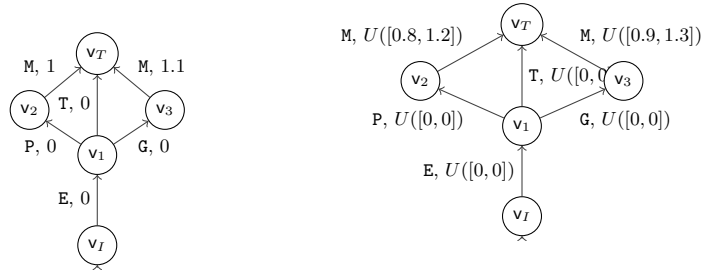

    \centering
    \begin{subfigure}{0.45\textwidth}
        \centering
        \ctikzfig{mine}
        \caption{RM: rewards are real numbers.}
        \label{subfig:classical-RM}
    \end{subfigure}
    \begin{subfigure}{0.45\textwidth}
        \centering
        \ctikzfig{mineStochastic}
        \caption{SRM: rewards are probability distributions.}
        \label{subfig:stochastic-RM}
    \end{subfigure}
    \caption{
        Classical reward machine for the mining example (left) and stochastic reward machine (right).
        Transitions are labeled with input-output pairs, where labels are given by propositional formulas encoding subsets of $2^P$.
        All states have a transition to $\mealyTerm$ with output $0$ on reading \texttt{T} (trap) (only depicted for $\mealyCommonState_1$).
        All remaining, missing transitions are self-loops with output $0$.
    }
    \label{figure:mining-machine}
\end{figure}

%% file: figures/mine_map.tex
\begin{tikzpicture}

\tikzset{mymarks/.style={-,opacity=0.5,
decoration={markings, mark=between positions 0 and 0.99 step #1 with {\draw[-latex](0,0)--(0.1,0);}},
postaction={decorate}
}
}

\draw[mymarks=0.5cm,blue] (-0.75,0.75) -- (0.25,0.75);
\draw[mymarks=0.5cm,blue] (0.25,0.75) -- (0.25,-0.25);
\draw[mymarks=0.5cm,blue] (0.25,-0.25) -- (-0.75,-0.25);

\draw[step=0.5cm,color=gray] (-0.5*4,-0.5*3) grid (0.5*4,0.5*3);
\matrix[matrix of nodes,nodes={inner sep=0pt,text width=0.5cm,align=center,minimum height=0.5cm}]{
            &            & \texttt{T} &            &            &            & \texttt{T} &            \\
 \texttt{E} &            & $\color{green}\bigcirc$           & \texttt{E} &            & \texttt{G} &            &            \\
 \texttt{T} &            &            & \texttt{T}           & \texttt{P}             &            &            &            \\
            &            & \texttt{M}            &            &  &            & \texttt{T} \\
            &            & \texttt{P} &            &            &            &            &            \\
            & \texttt{T} &            &            &            &            &            & \texttt{P} \\};

\end{tikzpicture}

%% file: sections/3_noisemodel.tex
\section{Stochastic reward machines}
\label{sec:noisemodel}

To capture stochastic, non-Markovian reward functions, we introduce the novel concept of stochastic reward machines.


\begin{definition}
    \label{def:srm}
    A \textbf{stochastic reward machine} (SRM) $\machine = (\mealyStates, \mealyInit, \mealyInputAlphabet, \mealyOutputAlphabet, \mealyTransition, \mealyOutput)$
    is defined by a finite, nonempty set $\mealyStates$ of states, 
    an initial state $\mealyInit \in \mealyStates$, 
    an input alphabet $\mealyInputAlphabet$, 
    a (deterministic) transition function $\mealyTransition : \mealyStates \times \mealyInputAlphabet \to \mealyStates$,
    an output alphabet $\mealyOutputAlphabet$ which is a finite set of cumulative distribution functions (CDFs),
    and output function $\mealyOutput : \mealyStates \times \mealyInputAlphabet \to \mealyOutputAlphabet$. 
    We define the size of $\machine$, denoted as $|\machine|$, to be $|\mealyStates|$ (i.e., the cardinality of the set $\mealyStates$).
\end{definition}

To define the semantics of an SRM, let $\trajectory{k}$ be a trajectory of an RL episode and $\mdpLabel_1 \ldots \mdpLabel_k$ the corresponding label sequence with $\mdpLabel_{i+1} = \rmLabelingFunction(s_{i}, a_{i+1},\allowbreak s_{i+1})$ for each $i \in \{ 0, \ldots, k-1 \}$.
The run of an SRM $\machine$ on the label sequence $\mdpLabel_1 \ldots \mdpLabel_n$ is then a sequence $\mealyCommonState_0 \mdpLabel_1 F_1 \mealyCommonState_1 \ldots \mealyCommonState_{n-1} \mdpLabel_n F_n \mealyCommonState_n$ where $\mealyCommonState_{i+1} = \mealyTransition(\mealyCommonState_i, \mdpLabel_{i+1})$ and $F_i = \mealyOutput(\mealyCommonState_i, \mdpLabel_{i+1})$ for all $i \in \{ 0, \ldots, k-1 \}$.
The sequence $F_1 \ldots F_k$ of CDFs is now used to determine the rewards that the agent receives: for each $i \in \{ 0, \ldots, k-1 \}$, the CDF $F_i$ is sampled and the resulting reward $r_i \in \mathcal R$ is returned.
This process results in a pair $(\lambda, \rho)$ with $\lambda = \mdpLabel_1 \ldots \mdpLabel_k$ and $\rho = r_1 \ldots r_k$, which we call a \emph{trace}.

We will often refer to SRM's output by the corresponding probability distribution.
We focus on continuous but bounded probability distributions, $\set{D_{1}([a_1,\allowbreak b_1]), \ldots, \distributionClosedN{n}{a_n}{b_n} }$ where $n \in \nat$ and $\distributionClosedN{i}{a}{b}$ is a distribution over the interval $[a, b]$.
It is not hard to see that the classical reward machines are a special case of Definition~\ref{def:srm}: one just has to set $a_i = b_i$ for each $i \in \{ 1, \ldots, n \}$, which will result in probability distributions that assign probability $1$ to a single real number.

Fig.~\ref{subfig:stochastic-RM} shows an SRM $\machineB$ for the mining example.
Note the difference to the classical RM of Fig.~\ref{subfig:classical-RM}: the transitions now output probability distributions instead of real values
(non-trivial uniform distributions are on transitions from $\mealyCommonState_2$ to $\mealyCommonState_T$ and from $\mealyCommonState_3$ to $\mealyCommonState_T$).
This difference allows us to capture the noise in the rewards of the example.
For example, the label sequence of our running example $(\emptyset, \set{\texttt{E}}, \emptyset, \set{\texttt{P}}, \emptyset, \emptyset, \set{\texttt{M}})$
will be transduced to a sequence of distributions. 
Sampling these distributions produces different reward sequences (e.g., $(0,0,0,0,0,0,0.95)$ and $(0,0,0,0,0,0,1)$).


Toro Icarte et al.~\cite{icarte_rms} introduced a version of Q-learning for reward machines called QRM, which assumes that the reward machine outputs deterministic rewards.
In the following, we examine if the guarantees of QRM can be retained when working with stochastic reward machines.

\subsection{QRM with stochastic reward machines}

In addition to specifying a learning task, reward machines help with the learning process itself. 
QRM assumes knowledge of a reward machine representation of environment reward and splits the Q-function update over all RM states by using transition outputs in lieu of empirical rewards.

For an MDP transition $\mdpTransitionTriplet$ with label $\mdpLabel$, QRM executes Q-function updates $\setOfQFunctions^{\mealyCommonState}(\mdpCommonState, \mdpCommonAction) \overset{\alpha}{\gets} r + \mdpDiscount \max_{\mdpCommonActionPrime} \setOfQFunctions^{\mealyCommonStatePrime}(\mdpCommonStatePrime, \mdpCommonActionPrime)$ for each reward machine state $\mealyCommonState$, with $\mealyCommonStatePrime$ being the succeeding state, $\alpha \in \reals$ the learning rate, and $r = \mealyOutput(\mealyCommonState, \mdpLabel)$ the reward due to reading label $\mdpLabel$ in state $\mealyCommonState$. 
These updates are equivalent to regular Q-learning updates in the cross-product of the original MDP and the reward machine.
The reward function is Markovian with respect to the resulting cross-product decision process.
As Q-learning also converges correctly for a \emph{stochastic} Markovian reward,
it is easy to see that QRM can find the optimal policy induced by an SRM by using samples $\mdpCommonRewardHat \sim \mealyOutput(\mealyCommonState, \mdpLabel)$ in the update rule.

We also remark that SRMs allow for a relaxed notion of equivalence.
As we will show in the following lemma, it is not necessary for two SRMs to be exactly equal in order to induce the same optimal policy.
We generalize the notion of exact functional equivalence (that is necessary for RMs) into the equivalence in expectation.

\begin{definition}
    SRMs $\machine$ and $\machineB$ are \textbf{equivalent in expectation} ($\machine \equivExpect \machineB$) if for every label sequence $\inputTrace = \labelSequence{k}$ we have $\E[\machine(\inputTrace)_i] = \E[\machineB(\inputTrace)_i]$ for every $1 \leq i \leq k$, that is if they output sequences of CDFs with equal expected values (where $\machine(\inputTrace)_i$ refers to the $i$-th CDF in the output sequence $\machine(\inputTrace)$).
\end{definition}

Lemma~\ref{lemma:same-policies} can simplify representation and inference of SRMs, allowing the algorithm to rely only on the expected values of the transitions in the inferred SRM.

\begin{lemma}
    \label{lemma:same-policies}
    If $\machine = (\mealyStates, \mealyInit, \mealyInputAlphabet, \mealyOutputAlphabet, \mealyTransition, \mealyOutput)$ and $\machineB = (\mealyStates ', \mealyInit ', \mealyInputAlphabet, \mealyOutputAlphabet ', \mealyTransition ', \mealyOutput ')$ are equivalent in expectation then they induce the same optimal policy over the same environment.
\end{lemma}


Now that it has been established that learning the optimal policy using stochastic reward machines is viable, the remaining question is whether one can drop the assumption that knowledge of the environment reward is accessible, and learn an SRM representation of it in conjunction with the policy (instead of assuming it to be given).

%% file: sections/4_methods.tex

\section{Inferring SRMs}
\label{sec:methods}

In this section, we show how to infer SRMs from data obtained through the agent's exploration.
In Section~\ref{sec:baseline}, we present a seemingly appealing baseline algorithm, and we explain its weaknesses.
We follow it by proposing \emph{Stochastic Reward Machine Inference} (SRMI)
as a better approach in Section~\ref{sec:sjirp}.

Both algorithms intertwine RL and learning of SRMs by starting with an initial hypothesis SRM and
\begin{enumerate}[label={(\arabic*)}]
    \item running QRM which generates a sequence of traces, and
    \item if there are traces contradicting the current hypothesis, inferring a new one.
\end{enumerate}

The steps repeat with the goal of recovering an SRM that captures the environment reward and using it to learn the optimal policy. QRM is performed in conjunction with the latest hypothesis. Traces which contradict the hypothesis are called \emph{counterexamples}.

Due to Lemma~\ref{lemma:same-policies}, the task is simplified to finding an SRM that is merely equivalent in expectation with environment rewards (instead of agreeing on exact distributions). We assume that a bound on noise dispersion $\dispersionParameter > 0$ is known (e.g., sensors come with pre-specified measurement error tolerance). 
Definition~\ref{def:epsilon-consistency} uses the $\dispersionParameter$ parameter to formalize the notion of consistency with a trace.

\begin{definition}
    \label{def:epsilon-consistency}
    A trace $(\inputTrace, \outputTrace) = (\labelSequence{k}, \rewardSequence{k})$ is \textbf{$\dispersionParameter$-consistent} with an SRM $\hypothesisRM$, which outputs a sequence of distributions $\hypothesisRM(\inputTrace) = \distributionSequence{k}$ if for all $1 \leq i \leq k$ we have $|\mdpCommonReward_i - \E[\outputDistribution_i]| \leq \epsilon_c$, i.e. if all of the observed rewards $\mdpCommonReward_i$ are plausible samples from $\hypothesisRM$.
\end{definition}

\sjirp can only recover an SRM representation of a noisy environment reward that meets an additional requirement, which we formalize in Assumption~\ref{assumption:non-containment}.
Informally, the assumption requires the noise from one reward distribution not to fully conceal the signal of a different one (unless they share means).
We are convinced this requirement is met in a large class of practical, real-world scenarios.

\begin{assumption}
    \label{assumption:non-containment}
    Let $\mealyOutputAlphabet = \set{\distributionClosedN{1}{a_1}{b_1}, \ldots, \distributionClosedN{n}{a_n}{b_n}}$ be the output alphabet of the environment SRM. Let $\dispersionParameter = \max_i \{b_i - a_i\} / 2$ be the noise dispersion bound known to the agent. We then assume that any two output distributions that can be covered with an $\dispersionParameter$-interval must have equal expectations: for all $1 \leq i, j \leq n$ and $\mean \in \reals$ we have $[a_i, b_i] \cup [a_j, b_j] \subseteq [\mean - \dispersionParameter, \mean + \dispersionParameter] \implies \E[\distributionClosed{a_i}{b_i}] = \E[\distributionClosed{a_j}{b_j}]$.
\end{assumption}

This assumption is satisfied in the Mining example.
The set $\set{\uniformClosed{0.1}{0.2},\allowbreak \uniformClosed{0}{1}}$ breaks Assumption~\ref{assumption:non-containment}: $\dispersionParameter$-intervals cannot distinguish these distributions, and they differ in expected values, so they must be distinguished.
The set $\set{\uniformClosed{0.1}{0.9},\allowbreak \uniformClosed{0}{1}}$ respects it, and distinguishing these distributions is unnecessary as they have equal expectations.

\subsection{Baseline algorithm}
\label{sec:baseline}

One may be tempted to repurpose existing techniques for inferring reward machines from a collection of traces. There are two important obstacles:

\begin{enumerate}
    \item Traces may be prefix-inconsistent: during exploration, the agent may encounter traces $(\labelSequence{m}, \rewardSequence{m})$ and $(\labelSequencePrime{n}, \rewardSequencePrime{n})$ s.t. for some $1 \leq i \leq \min\set{m, n}$ we have $\mdpLabel_1 \ldots \mdpLabel_i = \mdpLabelPrime_1 \ldots \mdpLabelPrime_i$ but $\mdpCommonReward_1 \ldots \mdpCommonReward_i \neq \mdpCommonRewardPrime_1 \ldots \mdpCommonRewardPrime_i$.
          The consequence is that no reward machine can capture both traces.
    \item Even if a collection of traces where noise is present is prefix-consistent, the (inferred) reward machine will tend to be impractically large because it will overfit noisy data.
\end{enumerate}

The baseline algorithm solves these problems by obtaining multiple samples for each trace in the counterexample set, and producing estimates for transition means \emph{before} inferring the structure of the reward machine. Starting with an initial hypothesis the following steps are repeated:

\begin{enumerate}[label={(\arabic*)}]
    \item run QRM which generates a sequence of traces
    \item when a counterexample is encountered, pause QRM and replay its trajectory until enough samples are collected
    \item preprocess the counterexample set so that multiple samples collected in (2) are collapsed into estimates for environment reward means
    \item use the deterministic RM inference method by Xu et al.~\cite{jirp} to infer the new minimal consistent hypothesis
\end{enumerate}

As knowledge of environment SRM structure is not assumed,
it is necessary to sample traces (instead of individual transitions).
The number of samples required in (2) is determined from $\dispersionParameter$ and an additional parameter for the minimal distance between two different transition means in the environment SRM. The preprocessing in (3) ensures (up to a confidence level) that different estimates for the same means are aggregated into one, and the result respects $\dispersionParameter$-consistency with the original sample set.

This approach seems to eliminate the two issues presented by stochastic rewards: since every prefix is sampled many times and then averaged and aggregated, there can be no prefix inconsistencies.
Furthermore, aggregation leaves little room for overfitting noise. 
However, the agent must be able to sample traces multiple times on demand.
As we show in section~\ref{sec:results}, this is costly, and sometimes even impossible. 

\subsection{Stochastic Reward Machine Inference}
\label{sec:sjirp}

In contrast to the baseline algorithm, 
\sjirp uses counterexamples to improve hypotheses \emph{immediately} by relying on a richer constraint solving method that is able to encode $\dispersionParameter$-consistency with the counterexample set directly. This removes the need for replaying trajectories.

As before, the task for the algorithm is to recover a minimal SRM that is equivalent in expectation to the true environment one, and use it to learn the optimal policy. 
Starting with an initial hypothesis SRM, the following steps are repeated:

\begin{enumerate}[label={(\arabic*)}]
    \item Run QRM and record all traces in a set $\setOfSeenTraces$ (Lines~\ref{line:qrm} to \ref{line:add-trace} in Algorithm~\ref{algorithm:sjirp}).
    \item When a counterexample is encountered, add it to the set $\counterexamples$ and attempt to make the current hypothesis $\dispersionParameter$-consistent with $\counterexamples$ by shifting its outputs (Lines~\ref{line:counterexample} to \ref{line:new-isomorphic}).
    \item If Step~(2) failed, solve a constraint problem to infer the new hypothesis (Line~\ref{line:new-inferred}).
    \item Compute the final mean estimates to correct outputs of the inferred hypothesis based on empirical rewards in $\setOfSeenTraces$ (Line~\ref{line:estimates}).
\end{enumerate}

The algorithm generates a sequence of hypothesis SRMs $\hypothesisRM_1 \hypothesisRM_2 \cdots$ and a sequence of counterexample sets $\counterexamples_1 \counterexamples_2 \cdots$ (with $\counterexamples_i \subset \counterexamples_j$ for all $i < j$) where $\hypothesisRM_j$ is consistent with $\counterexamples_{j-1}$ (and thus every $\counterexamples_i$ for $1 \leq i \leq j$). 
For simplicity, we assume noise distributions are symmetric, but \sjirp can easily be extended to cover asymmetric ones (see Appendix~\ref{sec:asymmetric-noise}). Hypothesis outputs are in the form $\distributionClosed{\mean - \dispersionParameter}{\mean + \dispersionParameter}$, where $\distributionClosed{a}{b}$ is a symmetric distribution over an interval, and $\mean \in \reals$ is the estimated mean of a particular transition.
Two SRMs are structurally isomorphic (Line~\ref{line:is-isomorphic}) if their underlying automata without the reward output are isomorphic in a graph-theoretic sense.


\begin{algorithm}
    \SetAlgoLined
    Initialize SRM $\hypothesisRM$ with a set of states $\mealyStates$\;
    Initialize a set of q-functions $\setOfQFunctions = \set{\qFunction^{\mealyCommonState} | \mealyCommonState \in \mealyStates}$\;
    Initialize $\counterexamples = \emptyset$ and $\setOfSeenTraces = \emptyset$\;
    \For{episode $n = 1, 2, \ldots$}{
        $(\inputTrace, \outputTrace, \setOfQFunctions) \gets$ QRM\_episode$(\hypothesisRM, \setOfQFunctions)$\; \nllabel{line:qrm}
        add $(\inputTrace, \outputTrace)$ to $\setOfSeenTraces$\; \nllabel{line:add-trace}
        
        \If{$\hypothesisRM$ not $\dispersionParameter$-consistent with $(\inputTrace, \outputTrace)$}{  \nllabel{line:counterexample}
            add $(\inputTrace, \outputTrace)$ to $\counterexamples$\; \nllabel{line:add-counterexample}
            \eIf{found SRM $\hypothesisRMZ$ isomorphic with $\hypothesisRM$ and $\dispersionParameter$-consistent with $\counterexamples$} { \nllabel{line:is-isomorphic}
                $\hypothesisRM ' \gets \hypothesisRMZ$\; \nllabel{line:new-isomorphic}
            }{
                infer $\hypothesisRM '$ from $\counterexamples$\;  \nllabel{line:new-inferred}
            }
            $\hypothesisRM \gets$ \texttt{Estimates} $(\hypothesisRM ', \setOfSeenTraces)$\; \nllabel{line:estimates}
            reinitialize $\setOfQFunctions$\;
        }
    }
    \caption{\sjirp}
    \label{algorithm:sjirp}
\end{algorithm}

There are many $\dispersionParameter$-consistent SRMs that can be returned by the constraint solving method in Line~\ref{line:new-inferred}, not necessarily having the best estimates for transition means. 
To correct for this, the function \emph{Estimates} assigns sets of observed rewards to each hypothesis transition by simulating its runs on traces in $\setOfSeenTraces$ and uses them to compute the final estimates.

Our algorithm categorized every counterexample as either Type~\easyCx or Type~\hardCx. A counterexample $(\inputTrace, \outputTrace)$ is of Type \easyCx with respect to $\hypothesisRM_i$ if there exists a graph-isomorphic SRM $\hypothesisRMZ$ that is consistent with $(\inputTrace, \outputTrace)$ and $\counterexamples_{i-1}$ (otherwise it is Type \hardCx). Then $\hypothesisRM_{i+1} = \hypothesisRMZ$ and $\counterexamples_i = \counterexamples_{i-1} \cup \{(\inputTrace, \outputTrace)\}$. Intuitively the current hypothesis can be "fixed" to become consistent with Type \easyCx counterexamples by shifting outputs without changing the structure of the SRM.
We now discuss how our algorithm handles counterexamples of Type~\hardCx.

\subsubsection{Inferring the structure and output range of a new hypothesis}

When a new Type \hardCx counterexample is encountered (i.e., outputs in the current hypothesis cannot be shifted to make it consistent), \sjirp infers a new hypothesis from the counterexamples, effectively solving the following task.

\begin{task}
    Given a set of traces $\counterexamples$ and dispersion $\dispersionParameter$, produce a minimal SRM that is $\dispersionParameter$-consistent with all traces in $\counterexamples$.
    \label{task:one}
\end{task}

We accomplish Task~\ref{task:one} by encoding it as a constraint problem in real arithmetic. 
Our encoding is an extension of the encoding used in the JIRP algorithm~\cite{jirp}.
Minimality is ensured by starting from machines of size $\rmSize = 1$, increasing the size by $1$ each time the constraint problem proves unsatisfiable, and returning the first successful result. 
For a given size, we use a collection of propositional and real variables from which one can extract an SRM, and constrain them so they (1) encode a valid SRM and (2) ensure that the SRM is consistent with the set $\counterexamples$.
More precisely, for size $n \in \mathbb N \setminus \{ 0 \}$, parameter $\dispersionParameter$, and counterexample set $\counterexamples$, we construct a formula $\smtFormula$ with the following two properties:
\begin{enumerate}[label=(\alph*)]
    \item $\smtFormula$ is satisfiable iff there exists an SRM $\hypothesisRM$ of size $\rmSize$ that is $\dispersionParameter$-consistent with every trace in $\counterexamples$. \label{item:property-one}
    \item Every satisfying assignment for variables in $\smtFormula$ contains sufficient information to construct a consistent SRM of size $n$. \label{item:property-two}
\end{enumerate}

The formula $\smtFormula$ is built using the following variables:
\begin{itemize}
    \item $d_{\mealyCommonStateP, \mdpLabel, \mealyCommonStateQ}$ are propositional variables, true iff $\mealyTransition(\mealyCommonStateP, \mdpLabel) = \mealyCommonStateQ$
    \item $o_{\mealyCommonState, \mdpLabel}$ are real variables that match the value of $\E[\mealyOutput(\mealyCommonState, \mdpLabel)]$
    \item $x_{\inputTrace, \mealyCommonState}$ are propositional variables encoding machine runs, true if the SRM arrives in state $\mealyCommonState$ upon reading label sequence $\inputTrace$
\end{itemize}

We use these variables to define constraints (\ref{formula:initState}) - (\ref{formula:epsilonConsistency}).
Formula~\ref{formula:initState} requires that at the beginning (after an empty sequence), the SRM is in the initial state. 
Formula~\ref{formula:transition} requires that the SRM transitions to exactly one state upon seeing a label. 

The remaining two formulas connect the SRM to the set $\counterexamples$
(we use symbol $\Pref(\counterexamples)$ for the set of prefixes of traces in $\counterexamples$).
Formula~\ref{formula:prefixRuns} connects seen prefixes to the transition function captured by variables $d_{\mealyCommonStateP, \mdpLabel, \mealyCommonStateQ}$.
Finally, Formula~\ref{formula:epsilonConsistency} ensures $\dispersionParameter$-consistency.

\begin{align}
    x_{\epsilon, \mealyInit} \wedge \bigwedge_{\mealyCommonState \in \mealyStates \setminus \set{\mealyInit}} \neg x_{\epsilon, \mealyCommonState} \label{formula:initState}                                                                                                                                                                                   \\
    \bigwedge_{\mealyCommonStateP \in \mealyStates} \bigwedge_{\mdpLabel \in \mealyInputAlphabet} \biggl[ \biggl[ \bigvee_{\mealyCommonStateQ \in \mealyStates} d_{\mealyCommonStateP, \mdpLabel, \mealyCommonStateQ} \biggr] \wedge \biggl[ \bigwedge_{\substack{\mealyCommonStateQ, \mealyCommonStateQ ' \in \mealyStates                                    \\ \mealyCommonStateQ \neq \mealyCommonStateQ '}} \neg d_{\mealyCommonStateP, \mdpLabel, \mealyCommonStateQ} \vee \neg d_{\mealyCommonStateP, \mdpLabel, \mealyCommonStateQ '} \biggr] \biggr] \label{formula:transition} \\
    \bigwedge_{(\inputTrace \mdpLabel, \outputTrace \mdpCommonReward) \in \Pref(\counterexamples)} \bigwedge_{\mealyCommonStateP, \mealyCommonStateQ \in \mealyStates} (x_{\inputTrace, \mealyCommonStateP} \wedge d_{\mealyCommonStateP, \mdpLabel, \mealyCommonStateQ}) \rightarrow x_{\inputTrace \mdpLabel, \mealyCommonStateQ} \label{formula:prefixRuns} \\
    \bigwedge_{(\inputTrace \mdpLabel, \outputTrace \mdpCommonReward) \in \Pref(\counterexamples)} \bigwedge_{\mealyCommonState \in \mealyStates} x_{\inputTrace, \mealyCommonState} \rightarrow |o_{\mealyCommonState, \mdpLabel} - \mdpCommonReward| \leq \dispersionParameter \label{formula:epsilonConsistency}
\end{align}

The formula $\smtFormula$ is defined as a conjunction of formulas (\ref{formula:initState}) - (\ref{formula:epsilonConsistency}).
One can easily see that properties \ref{item:property-one}, \ref{item:property-two} hold for $\smtFormula$ as there is a bijection between assignments for which $\smtFormula$ is true and consistent SRMs (modulo distributions).

Let $\hypothesisRM '$ be the SRM constructed from a model that satisfies the above constraints, with $\mealyOutput_{\hypothesisRM '}(\mealyCommonState, \mdpLabel) = \distributionClosed{o_{\mealyCommonState, \mdpLabel} - \dispersionParameter}{o_{\mealyCommonState, \mdpLabel} + \dispersionParameter}$. For every $(\inputTrace, \outputTrace) \in X$ of length $k$ and $1 \leq i \leq k$, due to (\ref{formula:epsilonConsistency}), we have $|\E[\hypothesisRM '(\inputTrace)_i] - \outputTrace_i| < \dispersionParameter$ and so $\hypothesisRM '$ is the new $\dispersionParameter$-consistent hypothesis. When $\hypothesisRM '$ is constructed by correcting for a type \easyCx counterexample, it is consistent by definition.

\subsubsection{Correcting output distributions in inferred SRMs}


As there can be many SRMs \epsconsistent{} with \counterexamples{} that are not equivalent in expectation, Task~\ref{task:one} need not have a unique solution.
To illustrate how nonuniqueness can prohibit correct convergence of hypothesis SRMs,
let \counterexamples{} contain samples from a single-state SRM $\environmentRM$ with two possible outputs,
$\distributionClosed{0}{1}$ on $\mdpLabel_1$ and $\distributionClosed{10}{100}$ on $\mdpLabel_2$ ($\dispersionParameter = (100-10)/2 = 45$).
For $\alpha \in \reals$ let $\hypothesisRM_{\alpha}$ be a single-state SRM
with two possible outputs, $\distributionClosed{\alpha - 45}{\alpha + 45}$ on $\mdpLabel_1$ and $\distributionClosed{10}{100}$ on $\mdpLabel_2$.
Then for all $-44 \leq \alpha, \beta \leq 45$ ($\alpha \neq \beta$) we have $\hypothesisRM_{\alpha} \not\equivExpect \hypothesisRM_{\beta}$, yet both are \epsconsistent{} with \counterexamples{}.
Thus, \sjirp must choose solutions to Task~\ref{task:one} so that any generated hypothesis sequence will converge to the same limit $\environmentRM$.

This is done in the \emph{Estimates} step of Algorithm~\ref{algorithm:sjirp},
which simulates runs of $\hypothesisRM '$ (the inferred solution to Task~\ref{task:one}) on \emph{consistent} traces in $\setOfSeenTraces$, yielding final estimates of means based on empirical rewards (shown in Algorithm~\ref{algorithm:estimates}).
The need to filter for consistency is due to the fact that the set of SRMs consistent with $\setOfSeenTraces$ is often strictly smaller than the set of those consistent with $\counterexamples$.
As $\counterexamples$ grows to capture more information about environment reward, the need for filtering recedes.

\begin{algorithm}
    \SetAlgoLined
    Initialize sets $\estimatesSet(\mealyCommonState, \mdpLabel)$ for states $\mealyCommonState \in \mealyStates$ and $\mdpLabel \in \mealyInputAlphabet$\;
    Initialize empty output function $\mealyOutput '$\;
    \For{$(\inputTrace, \outputTrace) \in \setOfSeenTraces$}{
        Skip $(\inputTrace, \outputTrace)$ if not $\dispersionParameter$-consistent with $\hypothesisRM$\;
        Simulate a run of $\hypothesisRM$ on $\inputTrace$ disregarding its outputs and record rewards from $\outputTrace$ in corresponding $\estimatesSet(\mealyCommonState, \mdpLabel)$ sets\;
        \For{$\mealyCommonState \in V$, $\mdpLabel \in \mealyInputAlphabet$}{
            $\mean ' \gets (\max \estimatesSet(\mealyCommonState, \mdpLabel) + \min \estimatesSet(\mealyCommonState, \mdpLabel))/2$\;
            Set $\mealyOutput '(\mealyCommonState, \mdpLabel) = U[\mean ' - \dispersionParameter, \mean ' + \dispersionParameter]$\;
        }
    }
    Return $\hypothesisRM$ with $\mealyOutput '$ as the output function\;
    \caption{Estimates}
    \label{algorithm:estimates}
\end{algorithm}

Using the midrange estimator for outputs in $\hypothesisRM$ ensures it remains consistent with the counterexamples. As new Type \easyCx counterexamples will always be found eventually, \emph{Estimates} runs infinitely often which guarantees convergence to correct means.

\subsubsection{Convergence to an optimal policy}

\begin{theorem}
    \label{theorem:convergence}
    Given Assumption~\ref{assumption:non-containment} on the output alphabet of the environment SRM and an $\epsilon$-greedy exploration strategy, \sjirp converges in the limit to an SRM that is equivalent in expectation to the true environment one.
\end{theorem}

The proof of Theorem~\ref{theorem:convergence} follows similar reasoning to the convergence proof for \jirp. We first establish that \sjirp does not revisit structurally isomorphic SRMs. As there are only finitely many such structures for a fixed maximal size, \sjirp "settles" in a final structure. Then Assumption~\ref{assumption:non-containment} guarantees that \emph{Estimates} will converge to the correct expectations.

\begin{corollary}
    \label{corollary:optimal-policy}
    \sjirp converges to an optimal policy in the limit.
\end{corollary}

Corollary~\ref{corollary:optimal-policy} follows from Theorem~\ref{theorem:convergence} due to Lemma~\ref{lemma:same-policies}, which guarantees that two SRMs that are equivalent in expectation induce the same optimal policy, and finally due to the fact that QRM with stochastic reward machines converges to an optimal policy.

\begin{figure}[!t]
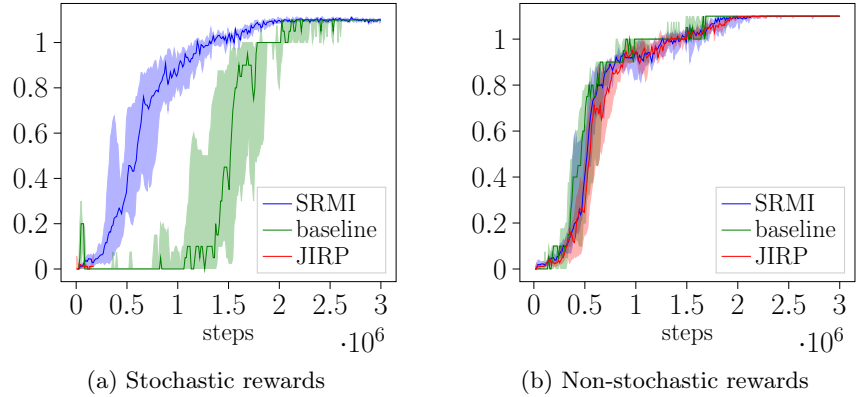

    \centering
    \begin{subfigure}{0.5\textwidth}
        \centering
        \input{images/results/mining.tex}
        \caption{Stochastic rewards}
        \label{figure:mining-results}
    \end{subfigure}%
    \begin{subfigure}{0.5\textwidth}
        \centering
        \input{images/results/nonstochastic_mining.tex}
        \caption{Non-stochastic rewards}
        \label{figure:mining-main-jirp}
    \end{subfigure}
    \caption{Results on the mining environment}
    \label{figure:results}
\end{figure}

\begin{figure}[b]
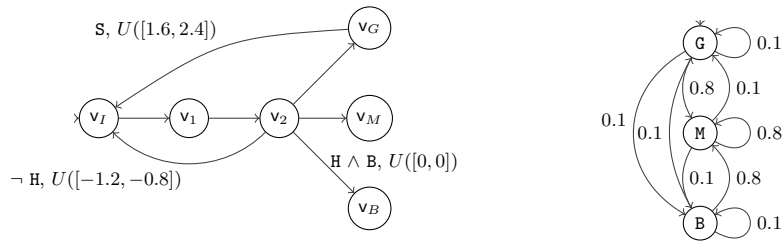

    \centering
    \begin{subfigure}{0.5\textwidth}
        \centering
        \ctikzfig{harvest}
        \caption{SRM: all missing transitions are self-loops with reward 0}
        \label{figure:harvest-srm}
    \end{subfigure}%
    \begin{subfigure}{0.5\textwidth}
        \centering
        \ctikzfig{harvest_mdp}
        \caption{MDP}
        \label{figure:harvest-mdp}
    \end{subfigure}
    \caption{Harvest environment}
    \label{figure:harvest-sketch-mdp}
\end{figure}

%% file: sections/5_results.tex

\section{Results}
\label{sec:results}

To assess the performance of \sjirp, we have implemented a Python\,3 prototype based on code by Toro Icarte et al.~\cite{icarte_rms}, which we have made publicly available\footnote{Source code available at \url{https://github.com/corazza/srm}.}.
To assess its performance, we compare \sjirp to the baseline algorithm and the \jirp algorithm for classical reward machines on two case studies: the mining example from Section~\ref{sec:background} and an example inspired by harvesting (which we describe shortly).


Our primary metric is the cumulative reward averaged over the last $100$ episodes. 
We conducted $10$ independent runs for each algorithm, using Z3~\cite{DBLP:conf/tacas/MouraB08} as the constraint solver.
All experiments were conducted on a  3\,GHz machine with 1.5\,TB RAM.

\subsubsection*{Mining}
Fig.~\ref{figure:mining-results} shows the comparison on the Mining environment.
The interquartile ranges for the reward are drawn as shaded areas, while the medians are drawn as solid lines.
For this case study, we have set the baseline algorithm to replay $20$ traces per counterexample.
As can be seen from the figure, \sjirp converges faster to an optimal policy (reward 1) than both the baseline algorithm and \jirp.
The latter times out because it is unable to deal with the noise properly and tries to infer larger and larger RMs.

Fig.~\ref{figure:mining-main-jirp} compares \sjirp, baseline, and \jirp on a non-stochastic version of the Mining environment using the reward machine of Fig.~\ref{subfig:classical-RM} to define the rewards.
All algorithms perform equally well. Thus, \sjirp does not incur a runtime penalty, even when used in non-stochastic settings.



\subsubsection*{Harvest}
The \emph{Harvest} environment represents a crop-farming cycle. The agent is rewarded for performing a sequence of actions, \texttt{P}, \texttt{W}, \texttt{H}, \texttt{S} (plant, water, harvest, sell), and penalized for breaking it. MDP states \texttt{G}, \texttt{M}, \texttt{B} (good, medium, bad) transition as given by the dynamics in Fig.~\ref{figure:harvest-mdp}. The labeling function $\rmLabelingFunction(\mdpCommonState, \mdpCommonAction, \mdpCommonState')$ returns the transition, effectively making trajectories $\trajectory{k}$ their own label sequences. The reward mean depends on the MDP state during the harvest action as shown in Fig.~\ref{figure:harvest-srm}.

\begin{figure}[t]
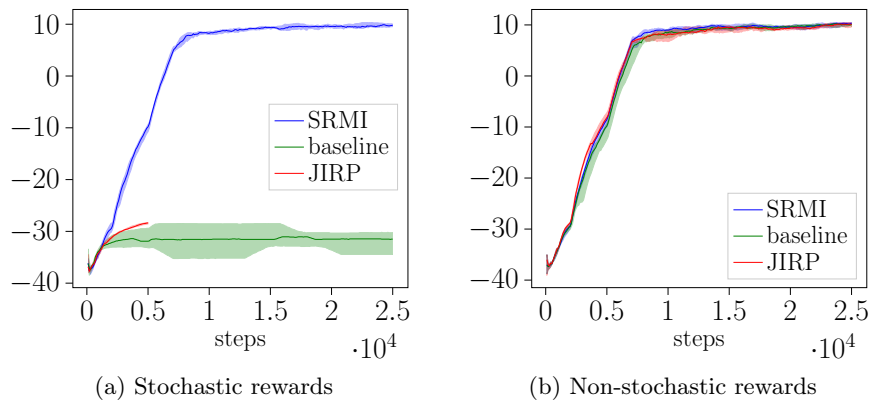

  \centering
  \begin{subfigure}{0.5\textwidth}
    \centering
    \input{images/results/harvest.tex}
    \caption{Stochastic rewards}
    \label{figure:harvest-stochastic}
  \end{subfigure}%
  \begin{subfigure}{0.5\textwidth}
    \centering
    \input{images/results/nonstochastic_harvest.tex}
    \caption{Non-stochastic rewards}
    \label{figure:harvest-nonstochastic}
  \end{subfigure}
  \caption{Results on the harvest environment}
  \label{figure:harvest-results}
\end{figure}





The Harvest example is well suited for showing the benefits of \sjirp over the baseline, because the probability that the agent will repeatedly see a given trajectory is very low.

Fig.~\ref{figure:harvest-stochastic} shows the comparison on the Harvest environment. \sjirp was successful in learning the optimal policy, while the baseline algorithm got stuck in collecting the required number of samples ($5$), and \jirp again timed out. In a modified Harvest environment, without noise, the algorithms do equally well (Fig.~\ref{figure:harvest-nonstochastic}).



%% file: sections/6_conclusion.tex
\section{Conclusion}
\label{sec:conclusion}

In this work we introduced Stochastic reward machines as a general way of representing non-Markovian stochastic rewards in RL tasks, and the \sjirp algorithm that is able to infer an SRM representation of the environment reward based on traces, and use it to learn the optimal policy. We have shown \sjirp is an improvement over prior methods.


%% file: appendix_sections/1_additional.tex

\section{Additional results}
\label{sec:additional-results}

Figure~\ref{figure:results-deep} shows the comparison between \sjirp and standard deep reinforcement learning algorithms, DDQN and DHRL (deep hierarchical reinforcement learning). The DDQN agent was provided the past $200$ labels as inputs to the network. The DHRL agent was based on an implementation provided in Toro Icarte et al.~\cite{icarte_rms}, who expanded the HRL options framework for use with reward machines (we note that DHRL in this case assumes access to SRM structure in order to generate options).

\sjirp significantly outperforms the alternatives which is in line with similar results for the non-stochastic setting~\cite{DBLP:conf/aaai/NeiderGGT0021}.

\begin{figure}[h]
    \centering
    \begin{subfigure}{0.5\textwidth}
        \centering
        \input{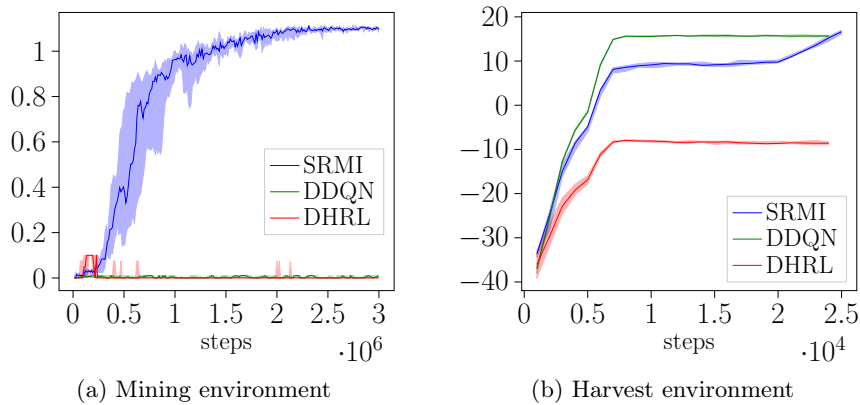}
        \caption{Mining environment}
        \label{figure:mining-deep}
    \end{subfigure}%
    \begin{subfigure}{0.5\textwidth}
        \centering
        \input{images/results/harvest_deep.tex}
        \caption{Harvest environment}
        \label{figure:harvest-deep}
    \end{subfigure}
    \caption{Comparison of \sjirp with DDQN and DHRL}
    \label{figure:results-deep}
\end{figure}

Note that, given a bounded symmetric distribution $D$, midpoint $\mu'$ is an unbiased estimator for $\mathbb{E}[D]$. We use it because it also works as a range estimate for $D$ that respects $\epsilon_c$-consistency, whereas the empirical average can break it. Asymmetric distributions complicate exposition, but in that case we do use the empirical average for estimating $\mathbb{E}[D]$ (see Algorithm~\ref{algorithm:estimates-asymmetric} in Appendix~\ref{sec:asymmetric-noise}).

%% file: images/results/harvest_deep.tex
\begin{tikzpicture}[scale=0.65]
  \begin{axis}[
      legend cell align={left},
      legend style={
          fill opacity=0.8,
          draw opacity=1,
          text opacity=1,
          at={(0.97,0.03)},
          anchor=south east,
          draw=white!80!black,
          font=\Large,
        },
      tick align=outside,
      tick pos=left,
      label style={font=\Large},
      tick label style={font=\LARGE},
      unbounded coords=jump,
      x grid style={white!69.0196078431373!black},
      xlabel={steps},
      xmin=-200, xmax=26200,
      xtick style={color=black},
      y grid style={white!69.0196078431373!black},
      ymin=-41.9557932875571, ymax=21,
      ytick style={color=black},
      try min ticks=6,
    ]
    \path [draw=blue, fill=blue, opacity=0.3, line width=0pt]
    (axis cs:1000,-33.4776839029404)
    --(axis cs:1000,-34.3341752203673)
    --(axis cs:2000,-28.2134537609499)
    --(axis cs:3000,-16.2275746936035)
    --(axis cs:4000,-10.3342009667016)
    --(axis cs:5000,-6.0695793182824)
    --(axis cs:6000,1.95690201246757)
    --(axis cs:7000,7.29426122299896)
    --(axis cs:8000,7.68469739975503)
    --(axis cs:9000,8.58792202227681)
    --(axis cs:10000,8.46617130585709)
    --(axis cs:11000,8.6999716631172)
    --(axis cs:12000,9.10678591399406)
    --(axis cs:13000,9.0249833750986)
    --(axis cs:14000,8.89780171753073)
    --(axis cs:15000,8.77278089123219)
    --(axis cs:16000,8.75423283961682)
    --(axis cs:17000,8.99214281215331)
    --(axis cs:18000,9.27717802858462)
    --(axis cs:19000,9.39618200070208)
    --(axis cs:20000,9.57189264752654)
    --(axis cs:21000,10.802451994585)
    --(axis cs:22000,12.1212482864986)
    --(axis cs:23000,13.2388315262805)
    --(axis cs:24000,14.5052889698947)
    --(axis cs:25000,16.0889509960776)
    --(axis cs:25000,16.9658657511427)
    --(axis cs:25000,16.9658657511427)
    --(axis cs:24000,15.6535439816086)
    --(axis cs:23000,14.1496852009133)
    --(axis cs:22000,12.6356644038544)
    --(axis cs:21000,11.4374753139366)
    --(axis cs:20000,10.2186127436962)
    --(axis cs:19000,10.1016209195215)
    --(axis cs:18000,10.0462362732652)
    --(axis cs:17000,10.1982146725622)
    --(axis cs:16000,9.45000694785748)
    --(axis cs:15000,9.44756336670289)
    --(axis cs:14000,9.78188851254322)
    --(axis cs:13000,9.61480872461998)
    --(axis cs:12000,9.60207552398239)
    --(axis cs:11000,9.6572567220847)
    --(axis cs:10000,9.65796728611238)
    --(axis cs:9000,9.40429301723693)
    --(axis cs:8000,8.92589741691665)
    --(axis cs:7000,8.52235555206536)
    --(axis cs:6000,3.60830992368642)
    --(axis cs:5000,-4.71964122288403)
    --(axis cs:4000,-8.25298004080319)
    --(axis cs:3000,-13.294989934717)
    --(axis cs:2000,-23.7833398355271)
    --(axis cs:1000,-33.4776839029404)
    --cycle;
    
    \path [draw=green!50.1960784313725!black, fill=green!50.1960784313725!black, opacity=0.3, line width=0pt]
    (axis cs:1000,-36.3049221155328)
    --(axis cs:1000,-38.0381799718422)
    --(axis cs:2000,-26.0309281268133)
    --(axis cs:3000,-13.0794174872451)
    --(axis cs:4000,-5.88336122045)
    --(axis cs:5000,-1.68727210324269)
    --(axis cs:6000,8.802256433513)
    --(axis cs:7000,14.8245063141015)
    --(axis cs:8000,15.5024827342533)
    --(axis cs:9000,15.4343397826083)
    --(axis cs:10000,15.4870191618786)
    --(axis cs:11000,15.5717991063952)
    --(axis cs:12000,15.7401337086019)
    --(axis cs:13000,15.6020251804545)
    --(axis cs:14000,15.6036549592276)
    --(axis cs:15000,15.5850309542345)
    --(axis cs:16000,15.4294761078451)
    --(axis cs:17000,15.5467389303332)
    --(axis cs:18000,15.541227771598)
    --(axis cs:19000,15.6869046118217)
    --(axis cs:20000,15.5058573308475)
    --(axis cs:21000,15.4415130417963)
    --(axis cs:22000,15.5224403634153)
    --(axis cs:23000,15.4218531350241)
    --(axis cs:24000,15.4876873230251)
    --(axis cs:24000,15.7546399337403)
    --(axis cs:24000,15.7546399337403)
    --(axis cs:23000,15.8644248372185)
    --(axis cs:22000,15.9200907395623)
    --(axis cs:21000,15.9036453861337)
    --(axis cs:20000,15.7805322492574)
    --(axis cs:19000,15.9039781716355)
    --(axis cs:18000,15.8410702853105)
    --(axis cs:17000,15.7529170957492)
    --(axis cs:16000,15.9327927612953)
    --(axis cs:15000,15.9481379282462)
    --(axis cs:14000,15.8891321124476)
    --(axis cs:13000,15.7804155838631)
    --(axis cs:12000,15.8214004716165)
    --(axis cs:11000,15.7709708590504)
    --(axis cs:10000,15.7813838824291)
    --(axis cs:9000,15.7109831326434)
    --(axis cs:8000,15.6828116605401)
    --(axis cs:7000,14.9602711719982)
    --(axis cs:6000,9.31011385423713)
    --(axis cs:5000,-1.09327960940262)
    --(axis cs:4000,-5.3932007771806)
    --(axis cs:3000,-12.417422774311)
    --(axis cs:2000,-24.8903444122333)
    --(axis cs:1000,-36.3049221155328)
    --cycle;
    
    \path [draw=red, fill=red, opacity=0.3, line width=0pt]
    (axis cs:1000,-32.85)
    --(axis cs:1000,-39.15)
    --(axis cs:2000,-31.85)
    --(axis cs:3000,-24.625)
    --(axis cs:4000,-20.325)
    --(axis cs:5000,-17.775)
    --(axis cs:6000,-11.7)
    --(axis cs:7000,-8.5)
    --(axis cs:8000,-8.1)
    --(axis cs:9000,-8.175)
    --(axis cs:10000,-8.425)
    --(axis cs:11000,-8.45)
    --(axis cs:12000,-8.575)
    --(axis cs:13000,-8.6)
    --(axis cs:14000,-8.375)
    --(axis cs:15000,-8.875)
    --(axis cs:16000,-8.675)
    --(axis cs:17000,-8.775)
    --(axis cs:18000,-8.775)
    --(axis cs:19000,-8.8)
    --(axis cs:20000,-8.775)
    --(axis cs:21000,-8.675)
    --(axis cs:22000,-8.825)
    --(axis cs:23000,-8.975)
    --(axis cs:24000,-8.95)
    --(axis cs:24000,-8.3)
    --(axis cs:24000,-8.3)
    --(axis cs:23000,-8)
    --(axis cs:22000,-8.1)
    --(axis cs:21000,-8.325)
    --(axis cs:20000,-8.175)
    --(axis cs:19000,-8.225)
    --(axis cs:18000,-8.25)
    --(axis cs:17000,-8.175)
    --(axis cs:16000,-8.1)
    --(axis cs:15000,-8.125)
    --(axis cs:14000,-8.15)
    --(axis cs:13000,-8)
    --(axis cs:12000,-8.2)
    --(axis cs:11000,-8.025)
    --(axis cs:10000,-8)
    --(axis cs:9000,-7.825)
    --(axis cs:8000,-7.825)
    --(axis cs:7000,-8.025)
    --(axis cs:6000,-10.725)
    --(axis cs:5000,-15.925)
    --(axis cs:4000,-17.875)
    --(axis cs:3000,-20.95)
    --(axis cs:2000,-26.65)
    --(axis cs:1000,-32.85)
    --cycle;
    
    \addplot [blue]
    table {%
        1000 -33.6219029881669
        2000 -24.7001991455783
        3000 -14.9997281081287
        4000 -8.60274136587366
        5000 -4.86303861956386
        6000 3.42132595574791
        7000 8.11263264697306
        8000 8.58873952711459
        9000 8.92496475690567
        10000 9.12706943988719
        11000 9.43136285617957
        12000 9.35733057962195
        13000 9.3672901350834
        14000 9.02210086089884
        15000 9.15722428380582
        16000 9.27111716424087
        17000 9.39056519334998
        18000 9.50600741153732
        19000 9.73912947578224
        20000 9.76183686182827
        21000 10.941314929349
        22000 12.2683333001698
        23000 13.6202348436829
        24000 15.2749273039358
        25000 16.634270166864
      };
    \addlegendentry{\sjirp}
    \addplot [green!50.1960784313725!black]
    table {%
        0 nan
        1000 -37.0084012269701
        2000 -25.3610251390192
        3000 -12.7447273317827
        4000 -5.72230062796195
        5000 -1.54091864923124
        6000 9.18304579038979
        7000 14.8944030366159
        8000 15.6280205041698
        9000 15.5781981195257
        10000 15.56221220798
        11000 15.6946127489678
        12000 15.7851478059513
        13000 15.6432252566883
        14000 15.7529364083316
        15000 15.7011201504393
        16000 15.765628448745
        17000 15.6457898890395
        18000 15.6303847299614
        19000 15.7425300441258
        20000 15.7204799763703
        21000 15.6394346223344
        22000 15.6629034057981
        23000 15.6624636688286
        24000 15.6967115749749
      };
    \addlegendentry{DDQN}
    \addplot [red]
    table {%
        0 nan
        1000 -36.05
        2000 -29.65
        3000 -22.85
        4000 -19.25
        5000 -16.85
        6000 -11.1
        7000 -8.35
        8000 -7.95
        9000 -8.1
        10000 -8.1
        11000 -8.2
        12000 -8.4
        13000 -8.4
        14000 -8.3
        15000 -8.35
        16000 -8.3
        17000 -8.5
        18000 -8.6
        19000 -8.7
        20000 -8.65
        21000 -8.5
        22000 -8.55
        23000 -8.6
        24000 -8.6
      };
    \addlegendentry{DHRL}
  \end{axis}
  
\end{tikzpicture}

%% file: appendix_sections/2_asymmetric.tex
\section{Asymmetric noise distributions}
\label{sec:asymmetric-noise}

In Algorithm~\ref{algorithm:sjirp} (\sjirp) we have assumed that rewards follow symmetric distributions, more precisely we assumed $\E[\distributionClosed{a}{b}] = (a + b)/2$ holds for all reward distributions in the environment SRM. Algorithms~\ref{algorithm:sjirp-asymmetric}~and~\ref{algorithm:estimates-asymmetric} show how to extend \sjirp towards dropping this assumption.

\begin{algorithm}
    \SetAlgoLined
    Initialize SRMs $\hypothesisRM$, $\hypothesisRMG$ with a set of states $\mealyStates$\;
    Initialize a set of q-functions $\setOfQFunctions = \set{\qFunction^{\mealyCommonState} | \mealyCommonState \in \mealyStates}$\;
    Initialize $\counterexamples = \emptyset$ and $\setOfSeenRewards = \emptyset$\;
    \For{episode $n = 1, 2, \ldots$}{
        $(\inputTrace, \outputTrace, \setOfQFunctions) \gets$ QRM\_episode$(\hypothesisRMG, \setOfQFunctions)$\;
        add $(\inputTrace, \outputTrace)$ to $\setOfSeenTraces$\;

        \If{$\hypothesisRM$ not $\dispersionParameter$-consistent with $(\inputTrace, \outputTrace)$}{
            add $(\inputTrace, \outputTrace)$ to $\counterexamples$\;
            \eIf{found SRM $\hypothesisRMZ$ isomorphic with $\hypothesisRM$ and $\dispersionParameter$-consistent with $\counterexamples$} {
                $\hypothesisRM ' \gets \hypothesisRMZ$\;
            }{
                infer $\hypothesisRM '$ from $\counterexamples$\;
            }
            $(\hypothesisRM, \hypothesisRMG) \gets$ \texttt{Estimates} $(\hypothesisRM ', \setOfSeenTraces)$\;
            reinitialize $\setOfQFunctions$\;
        }
    }
    \caption{\sjirp (asymmetric rewards)}
    \label{algorithm:sjirp-asymmetric}
\end{algorithm}

\begin{algorithm}
    \SetAlgoLined
    Initialize sets $\estimatesSet(\mealyCommonState, \mdpLabel)$ for states $\mealyCommonState \in \mealyStates$ and $\mdpLabel \in \mealyInputAlphabet$\;
    Initialize empty output functions $\mealyOutput '$, $\hat{\mealyOutput} $\;
    \For{$(\inputTrace, \outputTrace) \in \setOfSeenTraces$}{
        Skip $(\inputTrace, \outputTrace)$ if not $\dispersionParameter$-consistent with $\hypothesisRM$\;
        Simulate a run of $\hypothesisRM$ on $\inputTrace$ disregarding its outputs and record rewards from $\outputTrace$ in corresponding $\estimatesSet(\mealyCommonState, \mdpLabel)$ sets\;
        \For{$\mealyCommonState \in V$, $\mdpLabel \in \mealyInputAlphabet$}{
            $\mean ' \gets (\max \estimatesSet(\mealyCommonState, \mdpLabel) + \min \estimatesSet(\mealyCommonState, \mdpLabel))/2$\;
            $\hat{\mean} \gets \E[\estimatesSet(\mealyCommonState, \mdpLabel)]$\;
            Set $\mealyOutput '(\mealyCommonState, \mdpLabel) = U[\mean ' - \dispersionParameter, \mean ' + \dispersionParameter]$\;
            Set $\hat{\mealyOutput} (\mealyCommonState, \mdpLabel) = U[\hat{\mean} - \dispersionParameter, \hat{\mean} + \dispersionParameter]$\;
        }
    }
    Return $\hypothesisRM$, $\hypothesisRMG$ constructed from the input SRM with $\mealyOutput '$, $\hat{\mealyOutput}$ as output functions\;
    \caption{Estimates (asymmetric rewards)}
    \label{algorithm:estimates-asymmetric}
\end{algorithm}

Symmetric reward distributions allow one to use a single number to store information about both the mean of a hypothesis output, and its range. In other words, if $\mean$ is the mean estimate for some hypothesis output $\outputDistribution = \distributionClosed{a}{b}$ ($b - a = 2  \dispersionParameter$), one can safely set $\outputDistribution = \uniformClosed{\mean - \dispersionParameter}{\mean + \dispersionParameter}$. However, if rewards can be asymmetric then $\mean$ is not necessarily the center of $[a, b]$, and setting $\outputDistribution = \uniformClosed{\mean - \dispersionParameter}{\mean + \dispersionParameter}$ can break the detection of counterexamples.

In Algorithm~\ref{algorithm:sjirp-asymmetric} solve this problem by using an auxiliary hypothesis $\hypothesisRMG$. This is a purely technical choice, one may also implement a richer representation of SRMs that can store both pieces of information and thus eliminate the need for an auxiliary hypothesis.
The main hypothesis $\hypothesisRM$ has outputs $\outputDistribution = \distributionClosed{a}{b}$ that are always $\dispersionParameter$-consistent with $\counterexamples$, but do not necessarily follow the best available mean estimates. On the other hand, the auxiliary hypothesis $\hypothesisRMG$ eschews consistency in favor of using the best estimates. QRM is performed over the auxiliary hypothesis which captures information about the environment reward that is most immediately beneficial for learning the optimal policy. Hypothesis improvement (and counterexample detection) is performed over the main hypothesis.

The \emph{Estimates} function now returns both a new consistent hypothesis $\hypothesisRM$, and the new auxiliary hypothesis $\hypothesisRMG$. They are identical save for using different estimators. Using the arithmetic mean for $\hypothesisRMG$ ensures the SRM outputs used in QRM are unbiased when empirical rewards are sampled from asymmetric distributions. Using the mid-range for $\hypothesisRM$ preserves $\dispersionParameter$-consistency.

%% file: appendix_sections/3_proofs.tex

\section{Proofs}
\label{sec:proofs}

\subsection{Proof of Lemma~\ref{lemma:same-policies} and Theorem~\ref{theorem:convergence}}

To prove Lemma~\ref{lemma:same-policies} and Theorem~\ref{theorem:convergence} we first introduce an equivalence relation $\equivGraph$ over SRMs: for $\machine = (\mealyStates, \mealyInit, \mealyInputAlphabet, \mealyOutputAlphabet, \mealyTransition, \mealyOutput)$ and $\machineB = (\mealyStates ', \mealyInit ', \mealyInputAlphabet, \mealyOutputAlphabet ', \mealyTransition ', \mealyOutput ')$, $\machine \equivGraph \machineB$ iff they are graph-isomorphic, i.e. if there exists a bijection $\isomorphism : \mealyStates \to \mealyStates '$ s.t. $\isomorphism(\mealyInit) = \mealyInit '$ and $\mealyTransition ' (\isomorphism(\mealyCommonState), \mdpLabel)  = \isomorphism(\mealyTransition(\mealyCommonState, \mdpLabel))$ for all $\mealyCommonState \in \mealyStates$ and $\mdpLabel \in \mealyInputAlphabet$.
For isomorphism $\isomorphism : \machine \mapsto \machineB$ we define a function $\isomorphismFix : V \times \mealyInputAlphabet \to \reals$ that "fixes" the outputs of $\machine$ so that it is equivalent in expectation to $\machineB$ i.e. $\E[\mealyOutput(\mealyCommonState, \mdpLabel) + \isomorphismFix(\mealyCommonState, \mdpLabel)] = \E[\mealyOutput '(\isomorphism(\mealyCommonState), \mdpLabel)]$.

\subsubsection{Proof of Lemma~\ref{lemma:same-policies}}

Lemma~\ref{lemma:same-policies} states that machines that are equivalent in expectation induce the same optimal policy over an MDP.

\begin{proof}
    Let SRMs $\machine = (\mealyStates^A, {\mealyInit}^A, \mealyInputAlphabet, \mealyOutputAlphabet^A, \mealyTransition^A, \mealyOutput^A)$ and $\machineB = (\mealyStates^B, \mealyInit^B, \mealyInputAlphabet, \mealyOutputAlphabet^B, \mealyTransition^B,\allowbreak \mealyOutput^B)$ be equivalent in expectation.
    
    First, assume that $\machine$ and $\machineB$ are graph-isomorphic with $I : \machine \mapsto \machineB$.
    We now define a new SRM $\machine '$ which shares the state set and transition function with $\machine$, but for a given state $\mealyCommonState \in \mealyStates^A$ and label $\mdpLabel$, $\machine '$ outputs the corresponding distribution from $\machineB$:
    let $\mealyOutput(\mealyCommonState, \mdpLabel) = \mealyOutput^B(\isomorphism(\mealyCommonState), \mdpLabel))$ and define $\machine ' = (\mealyStates^A, \mealyInit^A, \mealyInputAlphabet, \mealyOutputAlphabet^B, \mealyTransition^A, \mealyOutput)$.
    Because $\machine \equivGraph \machineB$, we have $\machineB \equivGraph \machine '$.
    Furthermore, $\machineB$ and $\machine '$ are equivalent in the sense of CDFs that they output. Because $\machineB$ is just $\machine '$ with different objects for its states (given by $\isomorphism$), they must induce the same optimal policy.
    For $\machine$ and $\machine '$, the argument is that for every policy $\policy((\mdpCommonState, \mealyCommonState), \mdpCommonAction)$ the Q-functions $\qFunction_{\policy}^{\machine}$ and $\qFunction_{\policy}^{\machine '}$ (induced by $\machine$ and $\machine '$ respectively) in the cross-product MDP are identical by definition, as they only care about expected reward: the process of obtaining $\machine '$ by overwriting outputs in $\machine$ preserves expected values.
    To conclude, when $\machine$ and $\machineB$ are equivalent in expectation and graph-isomorphic, they induce the same optimal policy.
    
    The general case when $\machine$ and $\machineB$ are equivalent in expectation but not graph-isomorphic is obtained by dropping down to classical non-stochastic reward machines $\hat{\machine}$ and $\hat{\machineB}$ where equivalence in expectation implies equivalence. $\hat{\machine}$ ($\hat{\machineB}$) is constructed by overwriting outputs in $\machine$ ($\machineB$) with their expected values. Previous reasoning then establishes that $\machine$ and $\hat{\machine}$, since they are equivalent in expectation and graph isomorphic, induce the same optimal policy (similarly for $\machineB$ and $\hat{\machineB}$). Finally, $\machine$ and $\machineB$ induce the same optimal policy because $\hat{\machine}$ and $\hat{\machineB}$ do.
\end{proof}

\subsubsection{Inference of SRMs}

Before proving Theorem~\ref{theorem:convergence} we show that \sjirp finds the counterexamples necessary to improve its hypothesis. First, we formalize the notion of attainable trajectories in Definition~\ref{def:m-attainable}.

\begin{definition}
    \label{def:m-attainable}
    Let $\mdp = (\mdpStates, \mdpInit, \mdpActions, \mdpProb, \rmLabels, \rmLabelingFunction)$ be a labeled MDP and $\maxLengthEpisode \in \nat$ a natural number. We call a trajectory
    $\mdpTrajectory = \trajectory{k} \in (\mdpStates \times \mdpActions)^{*} \times \mdpStates$\; \textbf{$\maxLengthEpisode$-attainable} if (i) $k \leq \maxLengthEpisode$ and (ii) $\mdpProb(\mdpCommonState_i, \mdpCommonAction_i, \mdpCommonState_{i+1}) > 0$ for each $i \in \set{0,\ldots,k}$.
    Moreover, we say that a trajectory $\mdpTrajectory$ is attainable if there exists an $\maxLengthEpisode \in \nat$ such that $\mdpTrajectory$ is $\maxLengthEpisode$-attainable.
\end{definition}

We also refer to label sequences generated by attainable trajectories as attainable label sequences.

\begin{lemma}
    \label{lemma:m-attainable}
    \sjirp with episode length $\maxLengthEpisode$ explores every $\maxLengthEpisode$-attainable trajectory with a non-zero probability.
\end{lemma}

Lemma~\ref{lemma:m-attainable} is proved with induction over the length of trajectories. The proof for an equivalent result can be found in Xu et al.~\cite{jirp}.

\begin{corollary}
    \label{corollary:infinitely-often}
    \sjirp with episode length $\maxLengthEpisode$ explores every $\maxLengthEpisode$-attainable trajectory infinitely often.
\end{corollary}

Corollary~\ref{corollary:infinitely-often} follows from Lemma~\ref{lemma:m-attainable} because in the limit, the probability of not observing a trajectory falls to zero, so it will be observed again with probability $1$.

Lemma~\ref{lemma:hard-counterexamples} provides a lower bound on maximal episode length which allows \sjirp to observe sufficient evidence of \emph{structural} dissimilarity between the hypothesis and the environment SRM. It states that if there are attainable trajectories which can generate a counterexample of Type \hardCx, then there exists a Type \hardCx counterexample that "fits" into an episode.

\begin{lemma}
    \label{lemma:hard-counterexamples}
    For MDP $\mdp$ and environment SRM $\environmentRM$ (with dispersion $\dispersionParameter$), let $\hypothesisRM$ be the current hypothesis \sjirp is using to explore $\mdp$. Let $\mdpTrajectory = \trajectory{k}$ be an attainable trajectory in $\mdp$, $\inputTrace = \labelSequence{k}$ the corresponding label sequence, and $\environmentRM(\labelSequence{k}) = \distributionSequence{k}$ (considered as a sequence of random variables) the output of the environment SRM over $\inputTrace$. If $(\labelSequence{k}, \distributionSequence{k})$ is a plausible Type \hardCx counterexample, that is if for all $\hypothesisRMG \equivGraph \hypothesisRM$ (including $\hypothesisRM$) we have $\Pr(\hypothesisRMG \; \dispersionParameter \textrm{-inconsistent with } (\inputTrace, \outputTrace) \; | \; \outputTrace \sim \distributionSequence{k}) > 0$, then \sjirp with episode length $\maxLengthEpisode \geq 2^{|\mdp| + 1}(|\environmentRM| + 1) - 1$ will eventually observe a Type \hardCx counterexample.
\end{lemma}

\begin{proof}
    First assume that (1) there is a $k$-attainable label sequence $\inputTrace$, $k > \maxLengthEpisode \geq 2^{|\mdp| + 1}(|\environmentRM| + 1) - 1$, which can be sampled to generate a Type \hardCx counterexample for $\hypothesisRM$, and (2) there are no such $n$-attainable label sequences for $n \leq \maxLengthEpisode$. Let $\hypothesisRMG$ be an SRM s.t. $\hypothesisRMG \equivGraph \hypothesisRM$ and $\hypothesisRMG \equivExpect \environmentRM$ over all $\labelSequence{\maxLengthEpisode}$. $\hypothesisRMG$ exists because there are no Type \hardCx counterexamples for $\hypothesisRM$ over $n$-attainable label sequences for $n \leq \maxLengthEpisode$.
    
    Let $\environmentRM '$ ($\hypothesisRMG '$) be a deterministic RM s.t. $\environmentRM' \equivExpect \environmentRM$ ($\hypothesisRMG ' \equivExpect \hypothesisRMG$). $\environmentRM '$ and $\hypothesisRMG '$ agree on all label sequences up to length $\maxLengthEpisode$ ($\maxLengthEpisode < k$), but $\inputTrace$ is a $k$-attainable label sequence s.t. $\environmentRM '(\inputTrace) \neq \hypothesisRMG '(\inputTrace)$.
    However, this contradicts the automata-theoretic results in Xu et al.~\cite{jirp}.
\end{proof}

Lemma~\ref{lemma:hard-counterexamples} and Corollary~\ref{corollary:infinitely-often} together prove Corollary~\ref{corollary:will-find-counterexamples}.

\begin{corollary}
    \label{corollary:will-find-counterexamples}
    Given sufficient episode length, \sjirp almost surely finds Type \hardCx counterexamples if they exist.
\end{corollary}

Now that we have shown that \sjirp will continue to find "structural flaws" in its hypothesis if they exist, we will show that there is an endpoint of this iterative structural improvement.

More formally, Lemma~\ref{lemma:revisiting-classes} states that \sjirp does not revisit structurally isomorphic SRMs.

\begin{lemma}
    \label{lemma:revisiting-classes}
    Let $\hypothesisRM_1 \hypothesisRM_2 \cdots $ be the sequences of hypotheses generated by \sjirp. Let $[\hypothesisRM_{i_1}] [\hypothesisRM_{i_2}] \cdots $ (with $i_1 = 1$) be the sequence of $\equivGraph$-equivalence classes obtained from $\hypothesisRM_1 \hypothesisRM_2 \cdots $ by collapsing equivalent neighboring elemenents into their class. Then for all $j < k$ we have $[\hypothesisRM_{i_j}] \neq [\hypothesisRM_{i_k}]$.
    
\end{lemma}

\begin{proof}
    The sequence $\hypothesisRM_1 \hypothesisRM_2 \cdots $ is induced by the counterexample sequence $\counterexamples_1 \counterexamples_2 \allowbreak \cdots$ (with $\counterexamples_j \subset \counterexamples_k$ for all $j < k$) so that $\hypothesisRM_k$ is consistent with $\counterexamples_{k-1}$ for all $k > 1$ (and thus every $\counterexamples_j$ for $1 \leq j \leq k - 1$).
    
    By definition, a counterexample of Type $\easyCx$ induces a new hypothesis that remains in the same class, while a Type $\hardCx$ counterexample causes a "jump" to the next class. Assume $[\hypothesisRM_{i_j}] = [\hypothesisRM_{i_k}]$ for some $j < k$. Then there exists at least one intermediate class $[\hypothesisRM_{i_n}]$ ($j < n < k$) s.t. $[\hypothesisRM_{i_n}] \notin \set{[\hypothesisRM_{i_j}], [\hypothesisRM_{i_k}]}$ (otherwise $\hypothesisRM_{i_j} \cdots \hypothesisRM_{i_k}$ gets collapsed into $[\hypothesisRM_{i_j}]$). Let $\isomorphism : \hypothesisRM_{i_j} \mapsto \hypothesisRM_{i_k}$ be an isomorphism and $\isomorphismFix : {\mealyStates}^{\hypothesisRM_{i_j}} \times \mealyInputAlphabet \to \reals$ the function that "fixes" outputs of $\hypothesisRM_{i_j}$ so that it is equivalent in expectation to $\hypothesisRM_{i_k}$. By applying $\isomorphismFix$ to outputs of $\hypothesisRM_{i_j}$ we obtain an SRM that is $\dispersionParameter$-consistent with $\counterexamples_{i_k - 1}$ (as all output transitions from any $\hypothesisRM_i$ have the same range $2 \dispersionParameter$). But since $\counterexamples_{i_n - 1} \subset \counterexamples_{i_k - 1}$ this makes $\hypothesisRM_{i_n} \in [\hypothesisRM_{i_j}]$ as all new counterexamples between $\hypothesisRM_{i_j}$ and $\hypothesisRM_{i_n}$ (and indeed $\hypothesisRM_{i_k}$) were of Type $\easyCx$, a contradiction.
\end{proof}

Due to Corollary~\ref{corollary:will-find-counterexamples}, Lemma~\ref{lemma:revisiting-classes}, and the fact that there are only finitely many $\equivGraph$-equivalence classes for a fixed maximum number of states, we obtain Corollary~\ref{corollary:final-hypothesis}.

\begin{corollary}
    \label{corollary:final-hypothesis}
    There exists a final equivalence class $[\hypothesisFinal]$ which will never be left.
\end{corollary}

With $\hypothesisFinalSeq$ we denote the subsequence of $\hypothesisRM_1 \hypothesisRM_2 \cdots$ comprised of elements from $[\hypothesisFinal]$.

\subsubsection{Correcting output distributions}

The following two lemmas state that in the \emph{Estimates} step, eventually the reward sample sets $\estimatesSet(\mealyCommonState, \mdpLabel)$ contain samples from distributions with equal expectations (Lemma~\ref{lemma:equal-expectations}), and eventually the mean estimators $\meanEstimate$ computed from those sets are unbiased (Lemma~\ref{lemma:unbiased}).

\begin{lemma}
    \label{lemma:equal-expectations}
    For a machine $\hypothesisRM$ in the $[\hypothesisFinal]$ subsequence let $\setOfConsistentTraces_{\hypothesisRM}$ be the set of all traces in $\setOfSeenTraces$ consistent with $\hypothesisRM$, and $\estimatesSet_{\hypothesisRM}(\mealyCommonState, \mdpLabel) = \set{x \in \reals : (\inputTrace \mdpLabel, \outputTrace x) \in \Pref(C^{\hypothesisRM})$, $\hypothesisRM$ in state $\mealyCommonState$ after reading $\inputTrace}$ be the set of empirical rewards from traces in $\setOfConsistentTraces_{\hypothesisRM}$ for the transition from $\mealyCommonState$ on $\mdpLabel$. Let $\distributionSequence{n}$ denote the different distributions that elements of $\estimatesSet_{\hypothesisRM}(\mealyCommonState, \mdpLabel)$ are sampled from. Then $\E[\outputDistribution_1] = \cdots = \E[\outputDistribution_n]$.
\end{lemma}

\begin{proof}
    We prove Lemma~\ref{lemma:equal-expectations} in case when $n=2$, the general statement follows easily. Let $\estimatesSet_{\hypothesisRM}(\mealyCommonState, \mdpLabel)$ contain samples from $\outputDistribution_1$ and $\outputDistribution_2$ such that $\E[\outputDistribution_1] \neq \E[\outputDistribution_2]$. For distribution $d$ let $R(\outputDistribution) \subset
        \reals$ be the image of $\outputDistribution$. Let $[\alpha, \beta]$ be the shortest interval such that $[\alpha, \beta] \supseteq R(\outputDistribution_1) \; \cup \; R(\outputDistribution_2)$. If $\beta - \alpha \leq 2 \dispersionParameter$ Assumption~\ref{assumption:non-containment} is broken thus $\beta - \alpha > 2 \dispersionParameter$. But then $\hypothesisRM \notin [\hypothesisFinal]$ because a Type \hardCx counterexample exists, and due to Corollary~\ref{corollary:will-find-counterexamples} it will eventually be found. \sjirp eventually observes two reward prefixes $\rewardSequence{j}$ and $\rewardSequence{k}$, where $\commonReward_j \sim d_1$, $\commonReward_k \sim d_2$, and $|\commonReward_j - \commonReward_k| > 2 \dispersionParameter$, such that the same transition in $\hypothesisRM$ would have to account for both of them. The reward sequence that is observed latter comes from a Type \hardCx counterexample because it is impossible for any $\hypothesisRMG \equivGraph \hypothesisRM$ to be consistent with both traces that the two reward sequences come from as (due to the graph isomorphism) they would belong to the same transition in $\hypothesisRMG$ which breaks $\dispersionParameter$ consistency with $\counterexamples$.
\end{proof}

\begin{lemma}
    \label{lemma:unbiased}
    For a machine $\hypothesisRM$ in the $[\hypothesisFinal]$ subsequence let $\estimatesSet_{\hypothesisRM}(\mealyCommonState, \mdpLabel)$ and $\distributionSequence{n}$ be defined as in Lemma~\ref{lemma:equal-expectations}. Considering samples in $\estimatesSet_{\hypothesisRM}(\mealyCommonState, \mdpLabel) = \set{R_1, \ldots, R_k}$ as random variables (for all $1 \leq i \leq k$ there is a $1 \leq j \leq n$ s.t. $R_i \sim d_j$), let $\meanEstimate = (R_{(1)} + R_{(k)})/2$ (the midrange estimator used in \emph{Estimates}). Then $\E[\meanEstimate] = \E[\outputDistribution_j]$ for all $1 \leq j \leq n$.
\end{lemma}

\begin{proof}
    Without loss of generalization let $\estimatesSet_{\hypothesisRM}(\mealyCommonState, \mdpLabel)$ contain samples from a single environment distribution $\outputDistribution$.
    
    First note that all inconsistent traces in $\setOfSeenTraces$ (once the current hypothesis $\hypothesisRM \in [\hypothesisFinal]$) are necessarily of Type \easyCx. If there was a trace $\trace \in \setOfSeenTraces$ that was a Type \hardCx counterexample, its attainable label sequence would be encountered infinitely often and with probability $1$ another Type \hardCx counterexample would be reproduced, which is a contradiction with $\hypothesisRM \in [\hypothesisFinal]$.
    
    Finally, \emph{new} Type \easyCx counterexamples are never introduced into $\setOfSeenTraces$. Once such traces are encountered, they are added into $\counterexamples$ and the hypothesis is "fixed" to account for them. The remaining finite number of Type \easyCx counterexamples in $\setOfSeenTraces$ are all eventually accounted for, as their label sequences are explored and more extremal samples are collected.
    
    This establishes that once some finite amount of hypotheses in $[\hypothesisFinal]$ have been generated, $\setOfConsistentTraces_{\hypothesisRM} = \setOfSeenTraces$. Then then all $R_1 \cdots R_n$ in $\estimatesSet_{\hypothesisRM}(\mealyCommonState, \mdpLabel)$ are independent samples of $\outputDistribution$ and the lemma follows because the midrange estimator is unbiased.
\end{proof}

\subsubsection{Proof of Theorem~\ref{theorem:convergence}}

Theorem~\ref{theorem:convergence} states that the sequence $\hypothesisSeq$ generated by \sjirp converges to an SRM that is equivalent in expectation to the true environment one. As Corollary~\ref{corollary:final-hypothesis} established that \sjirp enters a final $\equivGraph$-class, we state the proof for the $\hypothesisFinalSeq$ subsequence.

\begin{proof}
    
    Without loss of generalization assume that for any two SRMs $\hypothesisRM$ and $\hypothesisRMG$ in $\hypothesisFinalSeq$ the graph-isomorphism $\isomorphism : \hypothesisRM \mapsto \hypothesisRMG$ is an identity, i.e. machines in $\hypothesisFinalSeq$ share the same state set $\mealyStates$ and $\mealyTransition_{\hypothesisRM}(\mealyCommonState, \mdpLabel) = \mealyTransition_{\hypothesisRMG}(\mealyCommonState, \mdpLabel)$ for all $\mealyCommonState \in \mealyStates$ and $\mdpLabel \in \mealyInputAlphabet$.
    To prove that $\hypothesisFinalSeq$ converges to the true environment reward machine we look at corresponding sequences $\estimatesSet^f_1(\mealyCommonState, \mdpLabel) \estimatesSet^f_2(\mealyCommonState, \mdpLabel) \cdots$ (for all $\mealyCommonState \in \mealyStates$, $\mdpLabel \in \mealyInputAlphabet$) of sets of empirical rewards from which \sjirp estimates the expectations for outputs of hypothesis SRMs.
    Since there will be no new Type \hardCx counterexamples, and because states and transitions of the hypotheses in the $\hypothesisFinalSeq$ are assumed to be fixed with $\isomorphism$, for each pair $(\mealyCommonState, \mdpLabel)$ this sequence of sample sets will only grow, and will contain only samples from a fixed set of distributions. This is a consequence of the fact that the environment SRM is fixed as well, so any attainable prefix of labels always induces the same pair of runs over the fixed hypotheses and the environment SRM. Due to Lemma~\ref{lemma:equal-expectations}, these distributions have equal expectations.
    Lemma~\ref{lemma:unbiased} shows the estimators computed from these sets are unbiased. Since there will be infinitely many Type \easyCx counterexamples, we conclude that the estimators converge to the correct expectations, that is $\hypothesisFinalSeq$ converges to an SRM that is equivalent in expectation to the true environment one.
\end{proof}